\documentclass{article}

\usepackage{graphicx}	
\usepackage{subcaption}
\usepackage{amsmath}	
\usepackage{amssymb}
\usepackage{amsmath}      
\usepackage{amssymb}      
\usepackage{amsthm}       
\usepackage{booktabs}
\usepackage{times}
\usepackage{microtype}
\usepackage{epsfig}
\usepackage{caption}
\usepackage{float}
\usepackage{placeins}
\usepackage{color, colortbl}
\usepackage{stfloats}
\usepackage{enumitem}
\usepackage{tabularx}
\usepackage{xstring}
\usepackage{multirow}
\usepackage{xspace}
\usepackage{url}
\usepackage{algorithm}
\usepackage{algpseudocode}
\usepackage{xcolor}
\usepackage{pifont}
\usepackage{wrapfig}
\usepackage[marginal]{footmisc}
\newtheorem{theorem}{Theorem}
\newtheorem{lemma}{Lemma}
\newtheorem{corollary}{Corollary}
\theoremstyle{remark}

\theoremstyle{definition}

\usepackage{etoc}
\etocdepthtag.toc{mtchapter}
\etocsettagdepth{mtchapter}{subsection}
\etocsettagdepth{mtappendix}{none}

\usepackage[normalem]{ulem}
\usepackage[colorlinks=true, citecolor=blue, linkcolor=blue, urlcolor=blue]{hyperref}

\newcommand{\name}{\textbf{FADRM}}
\newcommand{\MyRad}{\mathfrak{R}_n}
\newcommand{\R}{\mathbb{R}}

\newcommand{\Hcal}{\mathcal{H}}
\newcommand{\KL}{\operatorname{KL}}

\usepackage[preprint]{neurips_2025}
\setcitestyle{square,numbers,comma} 

\title{{FADRM}: \uline{F}ast and \uline{A}ccurate \uline{D}ata \uline{R}esidual \uline{M}atching for Dataset Distillation}

\author{
Jiacheng Cui$^{* 1}$, Xinyue Bi$^{*2}$ , Yaxin Luo$^1$, Xiaohan Zhao$^1$, Jiacheng Liu$^1$, Zhiqiang Shen$^{\dagger1}$ \\ [1ex]
    $^1$VILA Lab, MBZUAI \quad $^2$University of Ottawa \\[1ex]
    $^*$Equal Contribution \quad $^\dagger$Corresponding Author \\[1ex]
    \textbf{Code:}~\href{https://github.com/Jiacheng8/FADRM}{\textcolor{blue}{\textbf{FADRM (GitHub)}}}
}

\begin{document}

\maketitle

\begin{abstract}
Residual connection has been extensively studied and widely applied at the model architecture level. However, its potential in the more challenging data-centric approaches remains unexplored. In this work, we introduce the concept of {\bf \em Data Residual Matching} for the first time, leveraging data-level skip connections to facilitate data generation and mitigate data information vanishing. This approach maintains a balance between newly acquired knowledge through pixel space optimization and existing core local information identification within raw data modalities, specifically for the dataset distillation task. Furthermore, by incorporating optimization-level refinements, our method significantly improves computational efficiency, achieving superior performance while reducing training time and peak GPU memory usage by 50\%. Consequently, the proposed method {\bf  F}ast and {\bf  A}ccurate {\bf  D}ata {\bf  R}esidual {\bf  M}atching for Dataset Distillation (\name{}) establishes a new state-of-the-art, demonstrating substantial improvements over existing methods across multiple dataset benchmarks in both efficiency and effectiveness. For instance, with ResNet-18 as the student model and a 0.8\% compression ratio on ImageNet-1K, the method achieves 47.7\% test accuracy in single-model dataset distillation and 50.0\% in multi-model dataset distillation, surpassing RDED by +5.7\% and outperforming state-of-the-art multi-model approaches, EDC and CV-DD, by +1.4\% and +4.0\%.
\end{abstract}
\section{Introduction}
\label{sec:intro}

\begin{wrapfigure}{r}{0.5\textwidth}
    \centering
    \vspace{-0.5in}
    \includegraphics[width=1\linewidth]{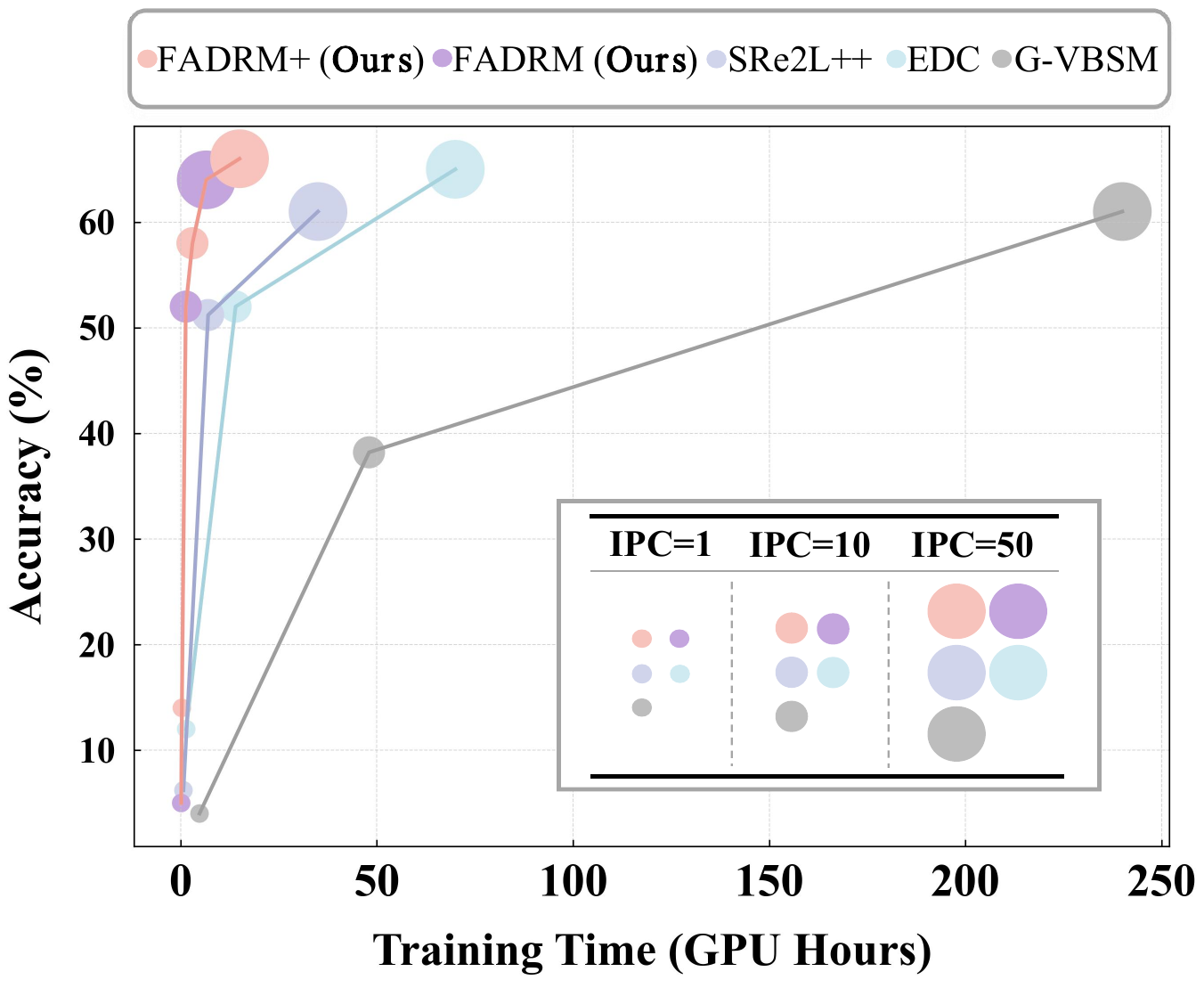}
    \vspace{-1.6em}
    \caption{Total training hours on a single RTX-4090 {\em vs.} test set accuracy, comparing prior state-of-the-art methods with our proposed framework (+ denotes multi-model distillation).}
    \label{fig:welcome}
    \vspace{-0.2in}
\end{wrapfigure}
In recent years, the computer vision and natural language processing communities have predominantly focused on model-centric research, driving an unprecedented expansion in the scale of neural networks. Landmark developments such as LLMs and MLLMs in ChatGPT~\cite{radford2018improving,achiam2023gpt}, Gemini~\cite{team2023gemini}, DeepSeek~\cite{liu2024deepseek} and other large-scale foundation models have shown the tremendous potential of deep learning architectures. However, as these models grow in complexity, the dependency on high-quality, richly informative datasets has become increasingly apparent, setting the stage for a paradigm shift towards data-centric approaches. Historically, the emphasis on building bigger and more complex models has often overshadowed the critical importance of the data. While model-centric strategies have delivered impressive results, they tend to overlook the benefits of optimizing data quality, which is essential for achieving higher performance with lower data demands. Recent advancements in data-centric research highlight the importance of improving information density, reducing the volume of required data, and expediting the training process of large-scale models, thus presenting a more holistic approach to performance enhancement.

\begin{figure*}[!t]
    \centering
    \includegraphics[width=1\textwidth]{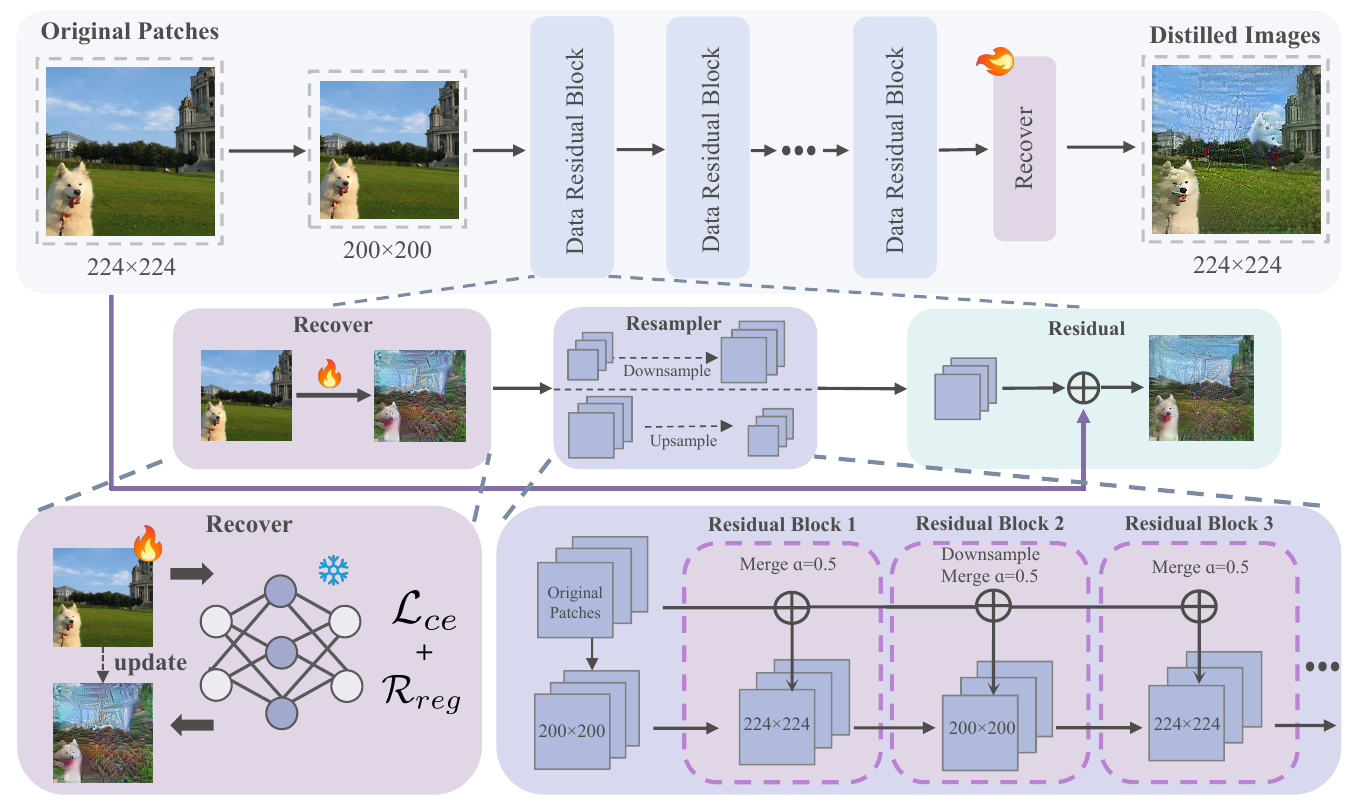}
    \caption{
    Overview of \name{}. It starts by downsampling the real data patches (both 1$\times$1 and 2$\times$2~\cite{RDED_2024} can be used as initialization and perform well in our experiments, meanwhile imposing downsampling to reduce cost). These downsampled images are subsequently processed through a series of proposed {\em Data Residual Blocks}. Each block utilizes a pretrained model to optimize the images within a predefined optimization budget, resamples them to a target resolution, and incorporates residual connections from the original patches via a mixing ratio $\alpha$. Finally, the images undergo an additional recovery stage, without residual connections, to produce the final distilled data. 
    }
    \label{fig:overview}
    \vspace{-.1in}
\end{figure*}

Within this evolving landscape, dataset distillation~\cite{wang2020datasetdistillation}, also called dataset condensation~\cite{kim2022dataset,zhaodataset,yin2024squeeze} has emerged as a pivotal area of research. The goal of dataset distillation is to compress large-scale datasets into smaller, highly informative subsets that retain the essential characteristics of the original data. This approach not only accelerates the training process of complex models but also mitigates the storage and computational challenges associated with massive datasets. Despite significant progress, many existing state-of-the-art methods in dataset distillation still struggle with issues related to scalability, generalization across diverse data resolutions, realism and robustness.

While residual connections have been well studied and widely implemented in the model architecture design field, primarily to prevent gradient vanishing and ensure effective feature propagation, their potential within data-centric paradigms remains largely unexplored. At the model level, residual connections help maintain the flow of gradients and enable deeper network architectures. In contrast, at the data level, similar connections can potentially prevent the loss of critical original dataset information and improve scalability and generalization across architectures during the data distillation process. This observation and design introduce a novel perspective on leveraging residual mechanisms beyond traditional model optimization, especially in the challenging domain of dataset optimization.

In this work, we introduce for the first time the concept of {\em Data Residual Matching} for dataset distillation. Our approach leverages data-level skip connections, a novel idea for data-centric task to prevent real data information vanishing in multi-block data synthesis architecture. We call our method {\bf F}ast and {\bf A}ccurate {\bf D}ata {\bf R}esidual {\bf M}atching (\name{}), which, as shown in Fig.~\ref{fig:overview}, employs a multi-resolution image recovery scheme that utilizes image resolution shrinkage and expansion in a residual manner, thereby capturing fine-grained details and facilitating the recovery of both global and local information. This balance between newly acquired knowledge through pixel space optimization and the preservation of existing core local information within raw data modalities marks a significant advancement in dataset distillation. By integrating these data-level residual connections, our approach enhances the generalization and robustness of the distilled datasets.

Exhaustive empirical evaluations of our proposed \name{} on CIFAR-100~\cite{cifar10}, Tiny-ImageNet~\cite{le2015tiny}, ImageNet-1k~\cite{deng2009imagenet} and its subset demonstrate that it not only accelerates the dataset distillation process by 50\% but also achieves superior accuracy that beats all previous state-of-the-art methods on both accuracy and generation speed. This approach effectively {\bf bridges the gap between model-centric and data-centric paradigms}, providing a robust solution to the challenges inherent in high-quality data generation. Our contributions in this paper are as follows: 
\begin{itemize}
\item We extend conventional residual connection from the model level to the data level area, and present for the first time a simple yet effective, theoretically grounded residual connection design for data generation to enhance data-centric task.
\item We introduce a novel dataset distillation framework based on the proposed {\em data residual matching}, incorporating multi-scale residual connections in data synthesis to improve both efficiency and accuracy.
\item Our approach achieves state-of-the-art results across multiple datasets, such as CIFAR-100, Tiny-ImageNet and ImageNet-1k, while being more efficient and requiring less computational cost than all previous methods.
\end{itemize}
\section{Related Work}
\label{sec:related_work}

\textbf{Dataset Distillation} aims to synthesize a compact dataset that retains the critical information of a larger original dataset, enabling efficient training while maintaining performance comparable to the full dataset. Overall, the matching criteria include {\em Meta-Model Matching}~\citep{wang2020datasetdistillation,nguyen2021dataset,loo2022efficient,zhou2022dataset,deng2022remember,he2024multisize}, {\em Gradient Matching}~\citep{zhaodataset,zhao2021dataset,lee2022dataset,kim2022dataset,zhou2024improve}, {\em Trajectory Matching}~\citep{MTT,cui2023scaling,chen2023dataset,guo2024lossless}, Distribution Matching~\citep{DBLP:conf/wacv/ZhaoB23, Wang_2022_CVPR,liu2022dataset,KFS,Sajedi_2023_ICCV,shin2024frequency,xue2024towards}, and {\em Uni-level Global Statistics Matching}~\citep{yin2024squeeze, GBVSM_2024,shao2024elucidating,yin2024dataset,cui2025dataset,xiao2024large}.
Dataset distillation on large-scale datasets has recently attracted significant attention from the community. For a detailed overview, it can be referred to the newest survey works~\citep{shang2025dataset,liu2025evolution} on this topic.
 
\noindent{\bf Efficient Dataset Distillation.} Several methods improve the computational efficiency of dataset distillation. TESLA~\cite{cui2023scaling} accelerates MTT~\citep{MTT} via batched gradient computation, avoiding full graph storage and scaling to large datasets. DM~\citep{DBLP:conf/wacv/ZhaoB23} sidesteps bi-level optimization by directly matching feature distributions. SRe$^2$L~\citep{yin2024squeeze} adopts a Uni-Level Framework that aligns synthetic data with pretrained model statistics. G-VBSM~\cite{GBVSM_2024} extends this by using lightweight model ensembles. EDC~\citep{shao2024elucidating} further boosts efficiency through real data initialization, accelerating convergence.

\noindent{\bf Residual Connection in Network Design.} Residual connections have played a pivotal role in advancing deep learning. Introduced in ResNet~\citep{he2016deep} to alleviate vanishing gradients, they enabled deeper networks by improving gradient flow. This idea was extended in Inception-ResNet~\citep{szegedy2017inception} through multi-scale feature integration, and further generalized in DenseNet~\citep{huang2017densely} via dense connectivity and feature reuse. Residual designs have also been central to Transformer architectures~\citep{vaswani2017attention}.

\section{Approach}
\label{sec:approach}

\noindent{\bf Preliminaries.}
Let the original dataset be denoted by $\mathcal{O} = \{(x_{i}, y_{i})\}_{i=1}^{|\mathcal{O}|}$, and let the goal of \textit{dataset distillation} be to construct a compact synthetic dataset $\mathcal{C} = \{(\tilde{{x}}_{j}, \tilde{y}_{j})\}_{j=1}^{|\mathcal{C}|}, \text{with } |\mathcal{C}| \ll |\mathcal{O}|$, such that the model \( f_{\theta_{\mathcal{C}}} \) trained on \( \mathcal{C} \) exhibits similar generalization behavior to the model \( f_{\theta_{\mathcal{O}}} \) trained on \( \mathcal{O} \). This objective can be formulated as minimizing the performance gap over the real data distribution:
\begin{equation}
    \operatorname*{arg\,min}_{\mathcal{C}, |\mathcal{C}|} 
    \sup_{(x, y) \sim \mathcal{O}} 
    \left| \mathcal{L} \left( f_{\theta_{\mathcal{O}}}(x), y \right) - 
    \mathcal{L} \left( f_{\theta_{\mathcal{C}}}(x), y \right) \right|
\label{Eq:preliminary}
\end{equation}
where the parameters \( \theta_{\mathcal{O}} \) and \( \theta_{\mathcal{C}} \) are obtained via empirical risk minimization:
\begin{equation}
    \theta_{\mathcal{O}} = \arg\min_\theta \mathbb{E}_{(x, y) \sim \mathcal{O}}[\mathcal{L}(f_\theta(x), y)], \quad
\theta_{\mathcal{C}} = \arg\min_\theta \mathbb{E}_{(\tilde{x}, \tilde{y}) \sim \mathcal{C}}[\mathcal{L}(f_\theta(\tilde{{x}}), \tilde{y})].
\end{equation}
\noindent The goal is to generate \( \mathcal{C} \) in order to maximize model performance with minimal data. Among existing methods, a notable class directly optimizes synthetic data without access to the original dataset, referred to as \emph{uni-level optimization}. While effective, this approach faces two key limitations: (1) progressive information loss during optimization, termed \emph{information vanishing}, and (2) substantial computational and memory costs for large-scale synthesis, limiting real-world applicability.

\noindent{\bf Information Vanishing.}
In contrast to images distilled using bi-level frameworks, the information content in images generated by uni-level methods (e.g., EDC~\cite{shao2024elucidating}) is fundamentally upper-bounded, as the original dataset is not utilized during synthesis (see Theorem~\ref{theorem:Bounded Information in BN-Aligned Synthetic Data}). As optimization progresses, the information density initially increases but eventually deteriorates due to the accumulation of local feature loss. This degradation leads to information vanishing (see Fig.~\ref{fig:info_vanish}), which significantly reduces the fidelity of the distilled images and limits their effectiveness in downstream tasks.

\begin{theorem} [Proof in Appendix~\ref{proof:Bounded Information in BN-Aligned Synthetic Data}]
\label{theorem:Bounded Information in BN-Aligned Synthetic Data}
Let $f_\theta$ be a pretrained neural network on original dataset $\mathcal{O}$ with fixed parameters and BatchNorm layers' mean and variance $\mathbf{BN}^{\text{RM}}$ and $\mathbf{BN}^{\text{RV}}$. Let $\tilde{{x}}$ denote an image optimized by minimizing the following loss:
   $ \mathcal{L}(\mathcal{C}) =  {\ell}_{CE}(f_{\theta}(\tilde{{x}}), \tilde{y})+ \lambda(\sum_{l} \left\| \mu_{l} (\tilde{{x}}) - \mathbf{BN}_{l}^{\text{RM}} \right\|_{2} + \sum_{l} \left\| \sigma_{l}^2 (\tilde{{x}}) - \mathbf{BN}_{l}^{\text{RV}} \right\|_{2})
   $.
Define:
\begin{equation}
    H\bigl(f_\theta\bigr)\;=\;\sup_{{x}\in\mathrm{supp}(\mathcal{O})}H\bigl(f_\theta({x})\bigr),
\end{equation}
as the maximum per‐sample Shannon entropy of the network's output.
Then, the mutual information between the optimized distilled dataset $\mathcal{C} = \{(\tilde{x}_{j}, \tilde{y}_{j})\}_{j=1}^{|\mathcal{C}|}$ and the original dataset $\mathcal{O}$ is bounded by:
\begin{equation}
    I(\mathcal{C} ; \mathcal{O}) \le |\mathcal{C}| H(f_{\theta}).
\end{equation} 
\end{theorem}
The insight of this theorem is that if the pretrained model $f_\theta$ is overly confident on all inputs (low maximum entropy), then $H(f_\theta)$ is small, and thus the distilled set, no matter how we optimize it, cannot encode a large amount of information about $\mathcal{O}$. 

\begin{figure}[htbp]
    \centering
    \begin{subfigure}[b]{0.47\textwidth}
        \includegraphics[width=\textwidth]{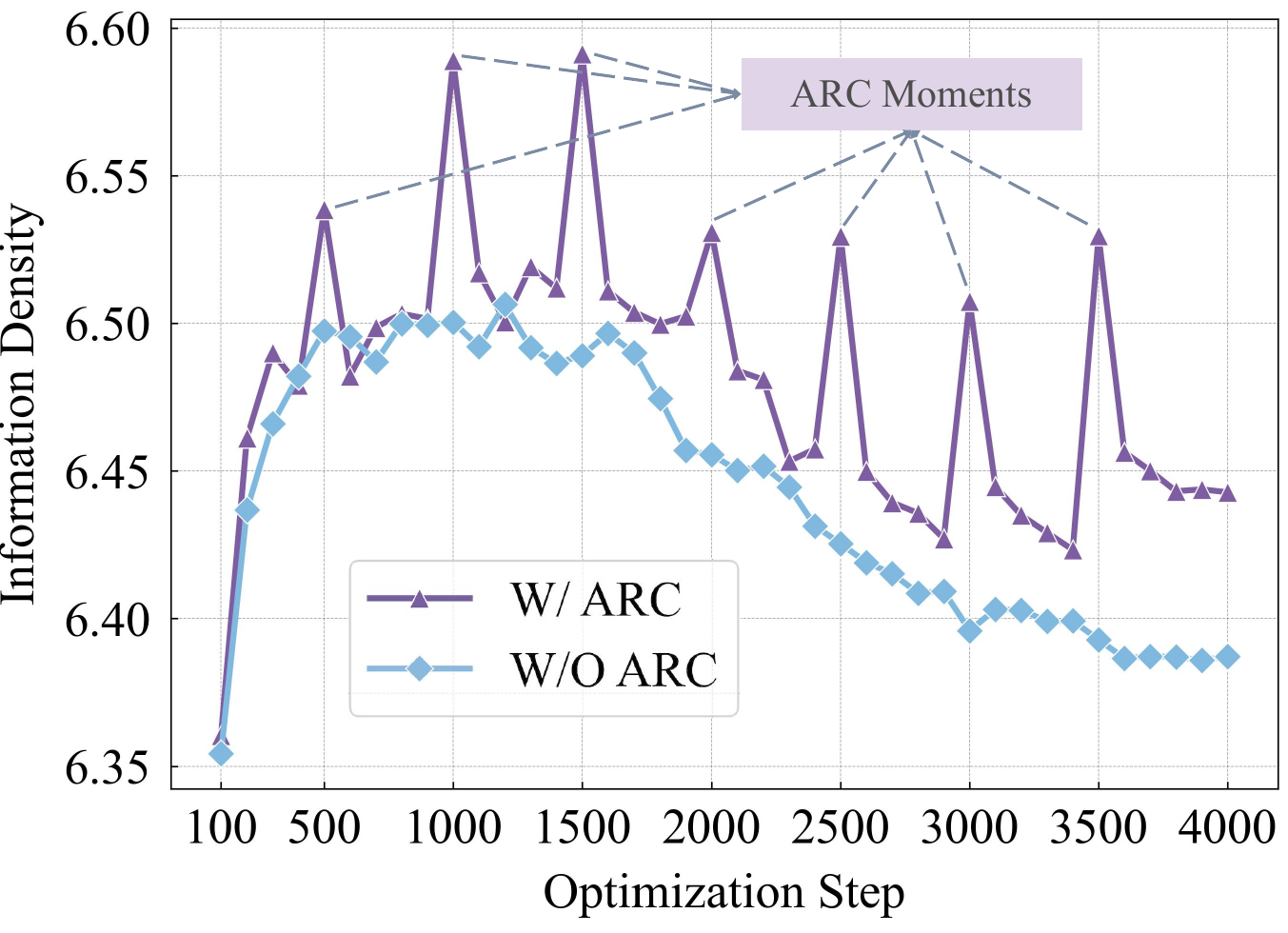}
    \end{subfigure}
    \hfill
    \begin{subfigure}[b]{0.47\textwidth}
        \includegraphics[width=\textwidth]{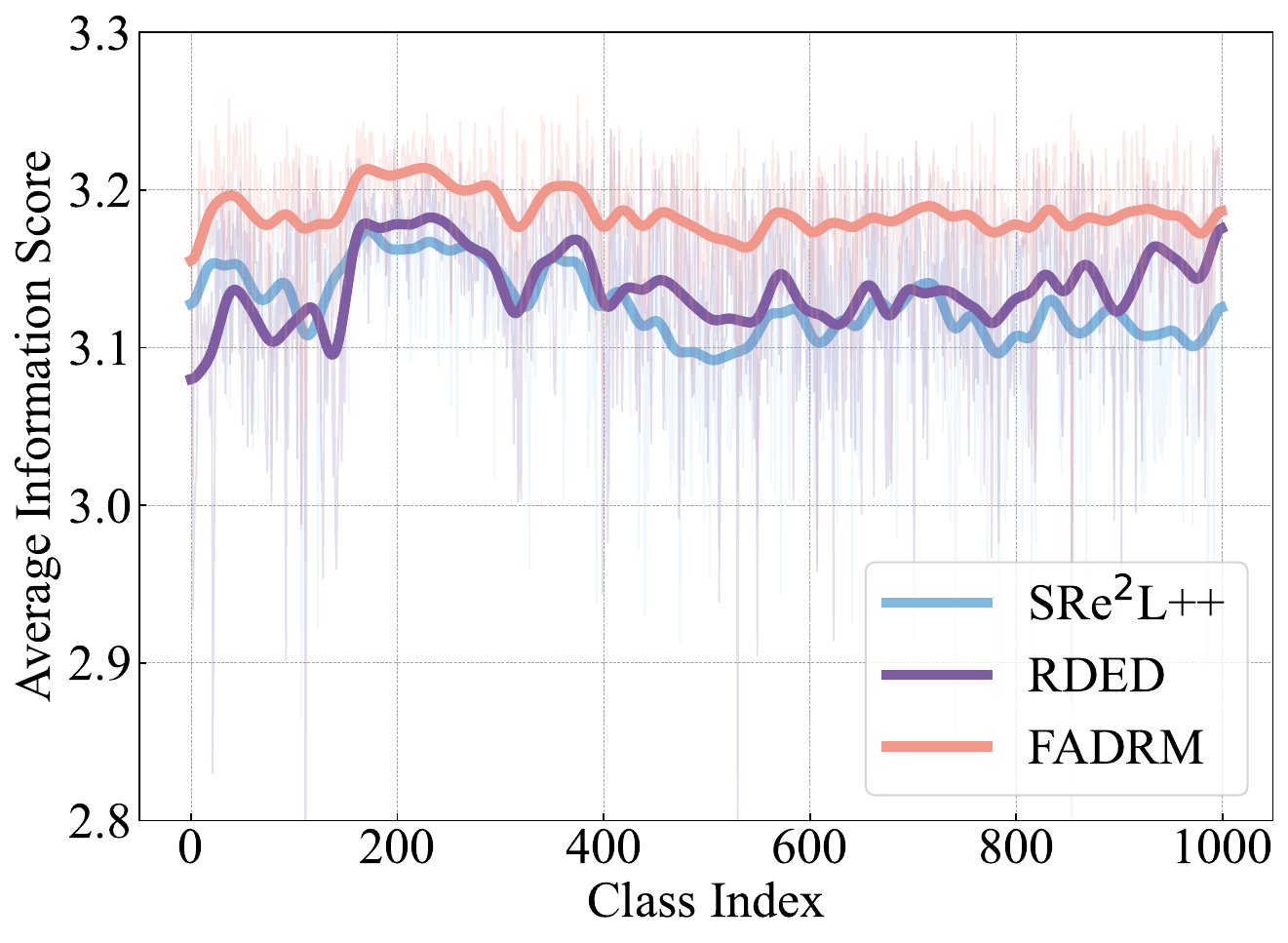}
    \end{subfigure}
    \caption{The above figures illustrate the phenomenon of \textit{Information Vanishing}. The \textbf{Left} Figure shows the evolution of information density across optimization steps, quantified through feature-level entropy using a pretrained ResNet-18~\cite{he2016deep}, comparing uni-level optimization (W/O \emph{ARC}) with our \name{} (W/ \emph{ARC}). The gray lines highlight the information density enhancement achieved through residual connection. The \textbf{Right} Figure shows the comparison of information scores (higher is better) across different classes, measured by \textit{pixel-level entropy}, among \name{}, SRe$^2$L++, and RDED. All experiments are conducted on a distilled ImageNet-1k dataset with IPC=10.}
    \label{fig:info_vanish}
\end{figure}

\noindent{\bf Computational Challenges.}
Although uni-level frameworks exhibit scalability to large-scale datasets, the overall time required to generate a large distilled dataset remains prohibitively expensive. As illustrated in Fig.~\ref{fig:welcome}, EDC~\cite{shao2024elucidating} requires nearly 70 hours to generate a 50 IPC distilled dataset, which limits its applicability in contexts involving repeated runs, large-scale data synthesis, or comprehensive empirical analysis. This motivates the need for more computationally efficient optimization strategies.

\subsection{Overview of \name{}}
The proposed \name{} framework, as illustrated in Fig.~\ref{fig:overview} and detailed in Algorithm~\ref{alg:FADRM}, addresses the limitations of existing uni-level optimization frameworks by integrating three proposed components: (1) \emph{MPT}: a mixed-precision training scheme that accelerates optimization and reduces computation by casting model parameters to lower-precision formats, (2) \emph{MRO}: a multiple resolution optimization that improves computational efficiency, and (3) \emph{ARC}: an adjustable embedded residual mechanism designed to seamlessly integrate essential features from the original dataset. This framework ensures both efficiency and generation fidelity in the optimization process.

\begin{algorithm}
\caption{\textbf{FADRM}: Residual Matching for Dataset Distillation}
\label{alg:FADRM}
\begin{algorithmic}[1]
\Require Recover model $f_\theta$, budget $\mathcal{B}$, real patches $\mathbf{P}_s$, merge ratio $\alpha$, downsampled resolutions $D_{\text{ds}}$, original resolutions $D_{\text{orig}}$, number of \emph{ARC}s $k$
\Ensure Distilled image $\tilde{x}_{\mathcal{B}}$
\State $b \gets \lfloor \mathcal{B}/(k{+}1) \rfloor$, \quad $\tilde{x}_0 \gets$ \textsc{Resample}($\mathbf{P}_s$, $D_{\text{ds}}$)
\For{$i=1$ \textbf{to} $k$}
    \For{$t=1$ \textbf{to} $b$}
        \State $\tilde{x}_{(i-1)b+t} \gets$ \textsc{GradStep}$(f_\theta, \tilde{x}_{(i-1)b+t-1})$ \Comment{Optimize $\tilde{x}$ to align the property of $f_\theta$}
    \EndFor
\State $\tilde{x}_{ib} \gets 
\begin{cases}
    \textsc{Resample}(\tilde{x}_{ib}, D_{\text{orig}}), & \text{if } \texttt{Shape}(\tilde{x}_{ib}) = D_{\text{ds}} \\
    \textsc{Resample}(\tilde{x}_{ib}, D_{\text{ds}}), & \text{otherwise}
\end{cases}$
\State $\tilde{x}_{ib} \gets \alpha \tilde{x}_{ib} + (1{-}\alpha)\cdot \textsc{Resample}(\mathbf{P}_s, \texttt{Shape}(\tilde{x}_{ib}))$
\EndFor
\For{$t=1$ \textbf{to} $\mathcal{B} - kb$}
    \State $\tilde{x}_{kb+t} \gets$ \textsc{GradStep}$(f_\theta, \tilde{x}_{kb+t-1})$
\EndFor
\State \Return $\tilde{x}_{\mathcal{B}}$
\end{algorithmic}
\end{algorithm}

\subsection{Mixed Precision Training for Data Generation}
Previous uni-level frameworks typically retain a fixed training pipeline, seeking efficiency through architectural or initialization-level changes. In contrast, we explicitly optimize the training process by incorporating Mixed Precision Training (MPT)~\cite{micikevicius2017mixed}. Specifically, we convert the model parameters \(\theta\) from FP32 to FP16 and utilize FP16 for both logits computation and cross-entropy loss evaluation. To preserve numerical stability and ensure accurate distribution matching, we retain the computation of the divergence to the global statistics (Appendix~\ref{sec:Optimization_Details}), as well as the gradients of the total loss with respect to \(\tilde{x}\) in FP32. By integrating \emph{MPT}, our framework significantly reduces both GPU memory consumption and training time by approximately 50\%, thereby significantly enhancing efficiency.

\subsection{Multi-resolution Optimization} 
Multi-Resolution Optimization (\emph{MRO}) enhances computational efficiency by optimizing images across multiple resolutions, unlike conventional methods that operate on a fixed input size.
Naturally, low-resolution inputs can reduce computational cost for the model, they often come at the expense of performance. To mitigate this, our method periodically increases the data resolution back at specific stages, resulting in a mixed-resolution optimization process, as illustrated in Fig.~\ref{fig:overview} (bottom-right).
 This approach is particularly beneficial for large-scale datasets (e.g., ImageNet-1K), where direct high-resolution optimization is computationally inefficient. Notably, optimization time scales significantly with input size for large datasets but remains stable for smaller ones (input size $\leq$ 64). Thus, \emph{MRO} is applied exclusively to large-scale datasets, as downscaling offers no efficiency gains for smaller ones.
Specifically, given an initialized image \( \mathbf{P}_s \in \mathbb{R}^{D_\text{orig} \times D_\text{orig} \times C} \), we first downsample it into a predefined resolution $D_{\text{ds}}$ utilizing bilinear interpolation (detailed in Appendix~\ref{Resampling}):
\begin{equation}
    \tilde{x}_0 = \text{Resample}(\mathbf{P}_s, D_{\text{ds}}), \quad D_{\text{ds}} < D_\text{orig}
\end{equation}

The downscaled images \( \tilde{x}_0\) are then optimized within a total budget \( b \), yielding the refined version \( \tilde{x}_b\). Subsequently, \( \tilde{x}_b\) are upscaled to their original dimensions:
\begin{equation}
    \tilde{x}_b = \text{Resample}(\tilde{x}_b,D_{\text{orig}})
\end{equation}

\noindent The upscaled image \( \tilde{x}_b \) is further optimized within the same budget \( b \) to recover information lost during the downscaling and upscaling processes. This iterative procedure (downscaling optimization and upscaling optimization) is repeated until the total optimization budget \(\mathcal{B}\) is exhausted. To ensure \emph{MRO} achieves efficiency gain without compromising quality, selecting an appropriate \(D_\text{ds}\) is critical. Excessively small \(D_\text{ds}\) risks significant information loss, degrading distilled data's quality, while overly large \(D_\text{ds}\) offers negligible efficiency benefits. Thus, \(D_\text{ds}\) must be carefully calibrated to balance efficiency and effectiveness. 

\noindent\textbf{Saved Computation by \emph{MRO}.}
Assume image-level convolution cost scales as \( \mathcal{O}(D^2 C) \). The baseline method performs all \( \mathcal{B} \) steps at full resolution \( D_{\text{orig}} \), yielding:
\begin{equation}
    \text{Cost}_{\text{baseline}} = \mathcal{B} \cdot \mathcal{O}(D_{\text{orig}}^2 C)
\end{equation}

\noindent \name{} performs \( k \) alternating-resolution stages of \( b = \lfloor \mathcal{B} / (k{+}1) \rfloor \) steps, with approximately half at downsampled resolution \( D_{\text{ds}} \). Let \( r = (D_{\text{ds}} / D_{\text{orig}})^2 \). The normalized cost is:
\begin{equation}
\frac{\text{Cost}_{\text{\emph{MRO}}}}{\text{Cost}_{\text{baseline}}} 
= 1 - \frac{b}{\mathcal{B}} \cdot \left( \left\lceil \tfrac{k}{2} \right\rceil (1 - r) \right)
\label{eq:mro-savings}
\end{equation}

Under fixed \( \mathcal{B},  k \), the cost ratio decreases linearly with \( 1 - r \). Smaller \( r \) (i.e., more aggressive downsampling) yields greater savings, but may compromise data fidelity. This trade-off highlights the role of \( r \) in balancing efficiency and representation quality during distillation.

\subsection{Adjustable Residual Connection}
\label{sec:arc}
\begin{wrapfigure}{r}{0.5\textwidth}
    \vspace{-.17in}
    \centering
   \includegraphics[width=0.5\columnwidth]{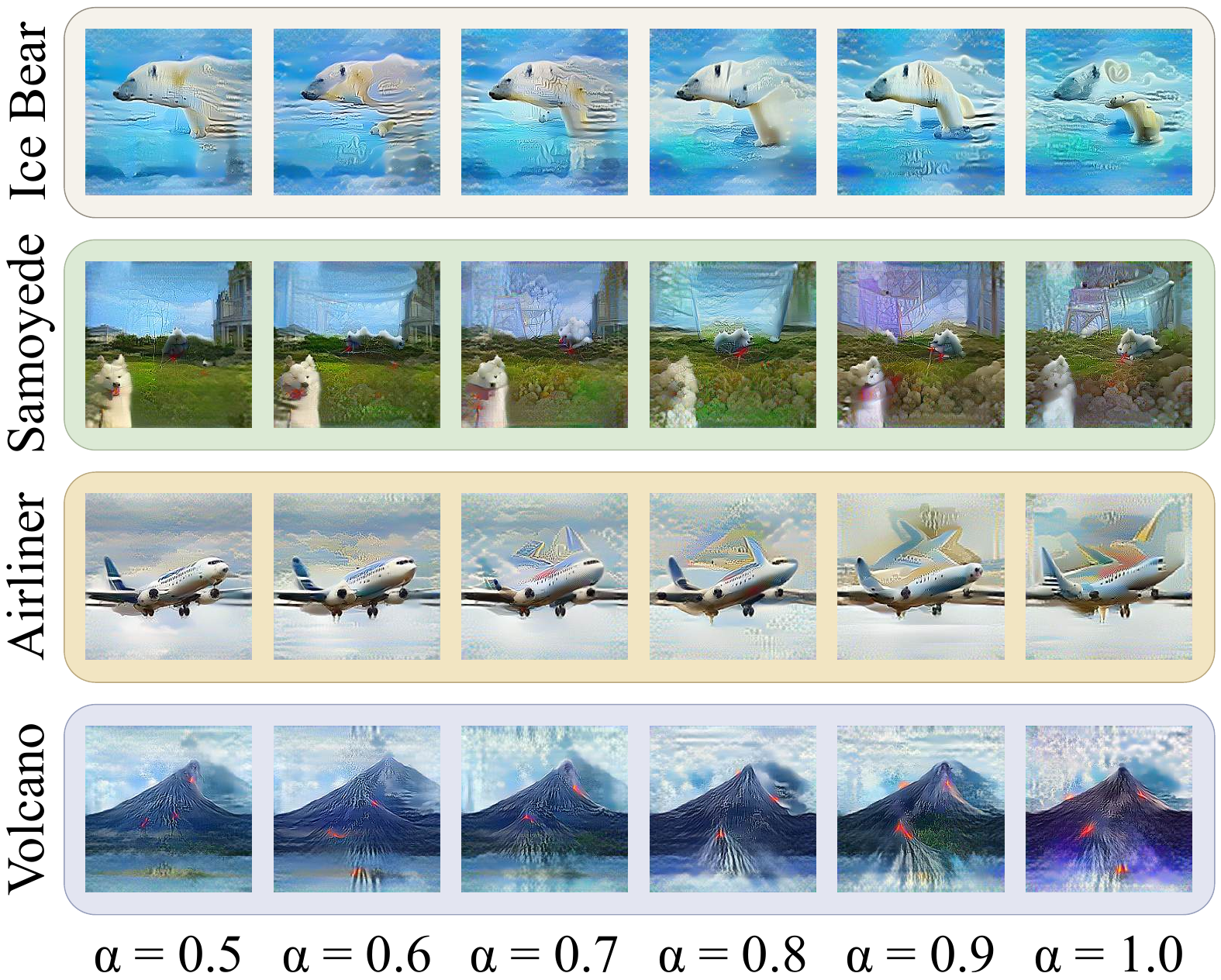} 
    \caption{Visualization of the distilled images with varying merge ratios using \name{}.}
    \label{fig:alpha_visualization}
\end{wrapfigure}
In uni-level optimization, the absence of the original dataset leads to information vanishing which significantly degrades the feature representation of the distilled dataset. To mitigate this issue, we introduce Adjustable Residual Connection (\emph{ARC}), a core mechanism that mitigates information vanishing (see Fig.~\ref{fig:info_vanish}) and improves the robustness of the distilled data (see Theroem~\ref{theorem:ARC improves robustness}). Essentially, \emph{ARC} iteratively fuses the intermediate optimized image \( \tilde{x}_t \in \mathbb{R}^{D_t \times D_t \times C} \)  at iteration \( t \) with the resized initialized data patches \(\tilde{\mathbf{P}}_t\), which contain subtle details from the original dataset. Formally, the update rule is defined as:
\begin{equation}
    \begin{aligned}
        \tilde{x}_t &= \alpha  \tilde{x}_t + (1-\alpha) \text{Resample}(\mathbf{P}_s, D_t)
    \end{aligned}
\end{equation}
where \( \alpha \in [0, 1] \) is a tunable merge ratio governing the contribution of original dataset information. A smaller \( \alpha \) strengthens the integration of details from \( \mathbf{P}_s \), whereas a larger \( \alpha \) prioritizes the preservation of the global features in the \( \tilde{x}_t \). This trend is visualized in Fig.~\ref{fig:alpha_visualization}. \emph{ARC} introduces a hyperparameter \( k \), which determines the frequency of residual injections. Given a total optimization budget of \( \mathcal{B} \), the training process is divided into \( k+1 \) segments, where residual connections occur after every \( b = \lfloor \mathcal{B} / (k+1) \rfloor \) iterations. The update follows:
\begin{equation}
    \tilde{\mathbf{P}}_{ib} = \operatorname{Resample} \big( \mathbf{P}_s, D_{ib} \big), \quad \tilde{x}_{ib} = \alpha \tilde{x}_{ib} + (1-\alpha) \tilde{\mathbf{P}}_{ib}.
\end{equation}
where \( i \in \{1, 2, \dots, k\} \) denotes the index of the residual injection stage, and \( D_{ib} \) indicates the spatial resolution of the intermediate image at the corresponding iteration \( t = ib \). The final phase consists of purely optimization without additional residual injections. Notably, \emph{ARC} performs a per-element weighted fusion of two image tensors with negligible overhead. With a complexity of \( \mathcal{O}(H_t W_t C) \), it scales linearly with the number of pixels and channels, making it well-suited for high-resolution data.

\begin{theorem}[Proof in Appendix~\ref{proof:ARC improves robustness}]
\label{theorem:ARC improves robustness}
Let \( \mathcal{H} \) be a class of functions \( h: \mathbb{R}^d \to \mathbb{R} \), and let  \( h \) be Lipschitz-continuous with constant \( L_h > 0 \), and the loss function $\ell$ be Lipschitz-continuous with constant \( L_l > 0 \) and bounded within a finite range \( [0, B] \). Consider: 1. Optimized perturbation added to the original data: \( \tilde{\mathcal{C}}^{res} = \{ \tilde{x}_{i}^{res}, \tilde{y}_{i}^{res} \}_{i=1}^n \). 2. residual injected dataset (FADRM): \( \tilde{\mathcal{C}}_{\mathrm{FADRM}} = \{ \tilde{x}_{i},\tilde{y}_{i} \}_{i=1}^n \). 3. patches selected from the original dataset: \( \mathcal{O} = \{ x_i , y_i\}_{i=1}^n \). 4. discrepancy \( \Delta := \frac{1}{n} \sum_{i=1}^n \| \tilde{x}_i^{res} - x_i \| \). Let \( h_{\mathrm{res}} \in \mathcal{H} \) denote the hypothesis trained on \( \tilde{\mathcal{C}}^{res} \), and \( h_{\mathrm{FADRM}} \in \mathcal{H} \) be trained on \( \tilde{\mathcal{C}}_{\mathrm{FADRM}} \). Define the corresponding empirical risks: $\widehat{\mathcal{L}}_{\mathrm{res}} := \frac{1}{n} \sum_{i=1}^n \ell(h_{\mathrm{res}}(\tilde{x}_i^{res}), \tilde{y}_i^{res}), \quad \widehat{\mathcal{L}}_{\mathrm{FADRM}} := \frac{1}{n} \sum_{i=1}^n \ell(h_{\mathrm{FADRM}}(\tilde{x}_{i}), \tilde{y}_i)$.
Suppose the following conditions hold:

\begin{equation}
    \label{eq:residual-upper-bound}
    \MyRad(\Hcal \circ \mathcal{O}) - \MyRad(\Hcal \circ \tilde{\mathcal{C}}^{res}) < -\frac{L_h\Delta(L_l+2B\alpha)}{2B}
\end{equation}

Where $\alpha$ is the merge ratio. Under the condition specified in Equation~\eqref{eq:residual-upper-bound}, the generalization bound of \( h_{\mathrm{FADRM}} \) is rigorously shown to be tighter than that of \( h_{\mathrm{res}} \), i.e.,
\begin{equation}
\widehat{\mathcal{L}}_{\mathrm{FADRM}} + 2B \cdot \MyRad(\mathcal{H} \circ \tilde{\mathcal{C}}_{\mathrm{FADRM}}) 
< \widehat{\mathcal{L}}_{\mathrm{res}} + 2B \cdot \MyRad(\mathcal{H} \circ \tilde{\mathcal{C}}^{res}).
\end{equation}

\end{theorem}

The insight of this theorem is that, when synthetic data is highly optimized and thus induces greater hypothesis complexity, combining it with the more structured and regular original data can lead to a tighter generalization bound, accounting for both empirical risk and Rademacher complexity.

\section{Experiments}
\label{sec:experiment}

\subsection{Datasets and Experimental Setup}

\noindent\textbf{Datasets.} 
We conduct experiments across datasets with varying resolutions, including CIFAR-100 (32$\times$32)~\citep{cifar10}, Tiny-ImageNet (64$\times$64)~\citep{le2015tiny}, ImageNet-1K (224$\times$224)~\citep{deng2009imagenet}, and their subsets.

\noindent\textbf{Baseline Methods.} 
To evaluate the effectiveness of our proposed framework, we conduct a comprehensive comparison against three state-of-the-art dataset distillation baselines. The first baseline is RDED~\cite{RDED_2024}, which selects cropped patches directly from the original dataset and is therefore categorized as involving full participation of the original data. The second method, EDC~\cite{shao2024elucidating}, retains a high degree of original data participation by optimizing selected patches with an extremely small learning rate, producing synthetic images that are close to the original samples. The third method, CV-DD~\cite{cui2025dataset}, aligns global BatchNorm statistics with sufficient optimization by updating initialization, resulting in minimal original data involvement despite initialization from real patches. These baselines exhibit varying degrees of original data involvement, providing a solid basis for evaluating \name{}.

\subsection{Main Results}

\noindent\textbf{Results Analysis.} As shown in Table~\ref{main_table}, our framework consistently achieves state-of-the-art performance across various settings. For instance, on ImageNet-1K with IPC=10 and ResNet-101 as the student model, the ensemble-enhanced variant \name{}+ attains an accuracy of 58.1\%, outperforming EDC and CV-DD by a substantial margin of +6.4\%. Notably, RDED underperforms \name{}, underscoring the limitations of relying solely on the original dataset without further optimization. Furthermore, CV-DD is inferior to \name{}+, highlighting the drawbacks of largely excluding original data during synthesis. Lastly, the consistent outperformance of \name{}+ over EDC validates the efficacy of our framework in harnessing original data via data-level residual connections.

\begin{table*}
    \centering
    \renewcommand{\arraystretch}{1}
    \resizebox{1.0\textwidth}{!}{
    \huge
    \begin{tabular}{cc|ccc>{\color{gray}}cc|ccc>{\color{gray}}cc|ccc>{\color{gray}}cc}
    \toprule[2.5pt] 
    & & \multicolumn{5}{c|}{ResNet18} & \multicolumn{5}{c|}{ResNet50} & \multicolumn{5}{c}{ResNet101} \\
    \cmidrule(lr){3-7} 
    \cmidrule(lr){8-12} 
    \cmidrule(lr){13-17} 
    Dataset & IPC (Ratio) & RDED & EDC & CV-DD & \name{} & \name{}+ & RDED & EDC & CV-DD & \name{} & \name{}+ & RDED & EDC & CV-DD & \name{} & \name{}+ \\
    \midrule
    \multirow{4}{*}{CIFAR-100}
    & 1 (0.2\%) & 17.1 & 39.7 & 28.3 & 31.8 & \textbf{40.6}&10.9 & 36.1 & 28.7 & 27.3 & \textbf{37.4} & 11.2 & 32.3 & 29.0 & 29.2 &\textbf{40.1}\\
    & 10 (2.0\%) & 56.9 & 63.7 & 62.7 & 67.4 & \textbf{67.9} & 41.6 & 62.1 & 61.5 & 66.5 & \textbf{67.4}& 54.1 & 61.7 & 63.8 & 68.3 & \textbf{68.9} \\
    & 50 (10.0\%) & 66.8 & 68.6 & 67.1 & 71.0 & \textbf{71.3} & 64.0 & 69.4 & 68.2 & 71.5 & \textbf{72.1} & 67.9 & 68.5 & 67.6 & 71.9& \textbf{72.1} \\
    & Whole Dataset & \multicolumn{5}{c|}{\cellcolor[gray]{0.9}\ \ \ 78.9} & \multicolumn{5}{c|} {\cellcolor[gray]{0.9}\ \ \ 79.9} & \multicolumn{5}{c} {\cellcolor[gray]{0.9}\ \ \ 79.5} \\
    \midrule
    \multirow{4}{*}{Tiny-ImageNet}
    & 1 (0.2\%) & 11.8 & 39.2 & 30.6 & 28.6 & \textbf{40.4}& 8.2 & 35.9 & 25.1 & 28.4 &\textbf{39.4} & 9.6 & 40.6 & 28.0 & 27.9& \textbf{41.9} \\
    & 10 (2.0\%) & 41.9 & 51.2 & 47.8 & 48.9& \textbf{52.8}& 38.4 & 50.2 & 43.8 & 47.3 & \textbf{53.7}& 22.9 & 51.6 & 47.4 & 47.8 &\textbf{53.6} \\
    & 50 (10.0\%) & 58.2 & 57.2 & 54.1 & 56.4 & \textbf{58.7}& 45.6 & 58.8 & 54.7 & 57.0 & \textbf{60.3}& 41.2 & 58.6 & 54.1 & 57.2 &\textbf{60.8}\\
    & Whole Dataset & \multicolumn{5}{c|} {\cellcolor[gray]{0.9}\ \ \ 68.9} & \multicolumn{5}{c|} {\cellcolor[gray]{0.9}\ \ \ 71.5} & \multicolumn{5}{c} {\cellcolor[gray]{0.9}\ \ \ 70.6} \\
    \midrule
    \multirow{4}{*}{ImageNette}
    & 1 (0.1\%) & 35.8 & - & 36.2 & 36.2 &\textbf{39.2}& 27.0 & - & 27.6 & 31.1 &\textbf{31.9}& 25.1 & - & 25.3 &  26.3&\textbf{29.3}\\
    & 10 (1.0\%) & 61.4 & - & 64.1 & 64.8 & \textbf{69.0} & 55.0 & - & 61.4 & 64.1& \textbf{68.1}& 54.0 & - & 61.0 & 61.9 &\textbf{63.7}\\
    & 50 (5.2\%) & 80.4 & - & 81.6 & 83.6 & \textbf{84.6}& 81.8 & - & 82.0 & 84.1 & \textbf{85.4} & 75.0 & - & 80.0 & 80.3 &\textbf{82.3} \\
    & Whole Dataset & \multicolumn{5}{c|} {\cellcolor[gray]{0.9}\ \ \ 93.8} & \multicolumn{5}{c|} {\cellcolor[gray]{0.9}\ \ \ 89.8} & \multicolumn{5}{c} {\cellcolor[gray]{0.9}\ \ \ 89.3} \\
    \midrule
    \multirow{4}{*}{ImageWoof}
    & 1 (0.1\%) & 20.8 & - & 21.4 & 21.0 &\textbf{22.8}& 17.8 & - & 19.1  &19.5  &\textbf{19.9}& 19.6 & - &  19.9& 20.0 &\textbf{21.8}\\
    & 10 (1.1\%)& 38.5 & - &49.3 & 44.5 & \textbf{57.3} & 35.2 & - & 47.8 & 44.9& \textbf{54.1}& 31.3 & - &42.6 & 40.4&\textbf{51.4}\\
    & 50 (5.3\%) & 68.5 & - & 71.9 & 72.3& \textbf{72.6}& 67.0& - & 71.2 & 71.0 & \textbf{71.7} & 59.1 & - & 69.9&70.3  &\textbf{70.6} \\
    & Whole Dataset & \multicolumn{5}{c|} {\cellcolor[gray]{0.9}\ \ \ 88.2} & \multicolumn{5}{c|} {\cellcolor[gray]{0.9}\ \ \ 77.8} & \multicolumn{5}{c} {\cellcolor[gray]{0.9}\ \ \ 82.7} \\
    \midrule
    \multirow{4}{*}{ImageNet-1k}
    & 1  (0.1\%) & 6.6 & 12.8 & 9.2 & 9.0 & \textbf{14.7} & 8.0 & 13.3 & 10.0 & 12.2 & \textbf{16.2}& 5.9 & 12.2 & 7.0 &  6.8 &\textbf{14.1} \\
    & 10 (0.8\%) & 42.0 & 48.6 & 46.0 & 48.4 & \textbf{50.9} & 49.7 & 54.1 & 51.3 & 54.5& \textbf{57.5} & 48.3 & 51.7 & 51.7 & 54.8 & \textbf{58.1}\\
    & 50 (3.9\%) & 56.5 & 58.0 & 59.5 & 60.1 & \textbf{61.2} & 62.0 & 64.3 & 63.9 & 65.4 & \textbf{66.9}& 61.2 & 64.9 & 62.7 & 66.0 & \textbf{67.0}\\
    & Whole Dataset & \multicolumn{5}{c|} {\cellcolor[gray]{0.9}\ \ \ 72.3} & \multicolumn{5}{c|} {\cellcolor[gray]{0.9}\ \ \ 78.6} & \multicolumn{5}{c} {\cellcolor[gray]{0.9}\ \ \ 79.8} \\
    \bottomrule[2.5pt] 
    \end{tabular}
    }
\caption{\textbf{Post-evaluation performance comparison with SOTA baseline methods.} All experiments follow the training settings established in EDC~\cite{shao2024elucidating}: 300 epochs for Tiny-ImageNet (IPC=10, 50), ImageNet-1K, and its subsets, and 1,000 epochs for CIFAR-100, Tiny-ImageNet (IPC=1). For fair comparison with single-model distillation (RDED) and ensemble-based methods (CV-DD, EDC), we evaluate both the single-model version (\name{} only utilized ResNet18~\cite{he2016deep} for distillation) and the ensemble-enhanced version (\name{}+). This evaluation strategy ensures equitable benchmarking while maintaining methodological consistency across all experiments.}
\label{main_table}
\end{table*}

\begin{table}[ht]
    \centering
    \begin{subtable}{0.44\linewidth}
        \centering
        \scriptsize
        \renewcommand{\arraystretch}{1.4}
        \resizebox{\linewidth}{!}{
        \begin{tabular}{lcc}
            \toprule
             Method & Time Cost (s) & Peak Memory (GB) \\
            \midrule
             SRe$^2$L++~\cite{cui2025dataset} & 2.52 & 5.3 \\
            \cellcolor[gray]{0.9}\name{} &\cellcolor[gray]{0.9} \textbf{0.47} & \cellcolor[gray]{0.9} \textbf{2.9} \\
            \midrule
             G-VBSM~\cite{GBVSM_2024} & 17.28 & 21.4 \\
             CV-DD~\cite{cui2025dataset} & 8.20 & 23.4 \\
             EDC~\cite{shao2024elucidating} & 4.99 & 17.9 \\
           \cellcolor[gray]{0.9} \name{}+ &\cellcolor[gray]{0.9} \textbf{1.09} & \cellcolor[gray]{0.9} \textbf{11.0} \\
            \bottomrule
        \end{tabular}
        }
    \end{subtable}
    \hfill
    \begin{subtable}{0.55\linewidth}
        \centering
        \renewcommand{\arraystretch}{1}
        \resizebox{\linewidth}{!}{  
        \begin{tabular}{lccccc}
        \toprule
        Model & \#Params & RDED & EDC & CV-DD & \name{}+ \\
        \midrule
        ResNet18~\cite{he2016deep} & 11.7M & 42.0 & 48.6 & 46.0 & \textbf{50.0} \\
        ResNet50~\cite{he2016deep} & 25.6M & 49.7 & 54.1 & 51.3 & \textbf{57.5} \\
        ResNet101~\cite{he2016deep} & 44.5M & 48.3 & 51.7 & 51.7 & \textbf{58.1} \\
        EfficientNet-B0~\cite{efficientnetv2} & 39.6M & 42.8 & 51.1 & 43.2 & \textbf{51.9} \\
        MobileNetV2~\cite{mobilenetv2} & 3.4M & 34.4 & 45.0 & 39.0 & \textbf{45.5} \\
        ShuffleNetV2-0.5x~\cite{zhang2018shufflenet} & 1.4M & 19.6 & 29.8 & 27.4 & \textbf{30.2} \\
        Swin-Tiny~\cite{liu2021swin} & 28.0M & 29.2 & 38.3 & -- & \textbf{39.1} \\
        Wide ResNet50-2~\cite{he2016deep} & 68.9M & 50.0 & -- & 53.9 & \textbf{59.1} \\
        DenseNet121~\cite{huang2017densely} & 8.0M & 49.4 & -- & 50.9 & \textbf{55.4} \\
        DenseNet169~\cite{huang2017densely} & 14.2M & 50.9 & -- & 53.6 & \textbf{58.5} \\
        DenseNet201~\cite{huang2017densely} & 20.0M & 49.0 & -- & 54.8 & \textbf{59.7} \\
        \bottomrule
        \end{tabular}
        }
    \end{subtable}
    \vspace{1em}
    \caption{\textbf{Left:} Efficiency comparison between various optimization-based methods and our approach when distilling ImageNet-1k. The time cost is measured in seconds, representing the duration required to generate a single image on a single RTX 4090 GPU. \textbf{Right:} Top-1 accuracy (\%) on ImageNet-1K for cross-architecture generalization with IPC=10.
    }
    \vspace{-.2in}
    \label{tab:combined_tab_efficiency_cross}
\end{table}

\noindent\textbf{Efficiency Comparison.} Table~\ref{tab:combined_tab_efficiency_cross} (Left) highlights the superior efficiency of our framework compared to existing Uni-level frameworks. Bi-level frameworks are excluded from this comparison due to their inherent limitations in scalability for large-scale datasets. Specifically, \name{}+ achieves a reduction of 3.9 seconds per image in optimization time compared to EDC~\cite{shao2024elucidating}, culminating in a total computational saving of 54 hours when applied to the 50 IPC ImageNet-1K dataset. Similarly, \name{} demonstrates a 28.5 hours reduction in training time relative to SRe$^2$L++ for the same task. Additionally, our framework significantly reduces peak memory usage compared to other frameworks, enabling efficient dataset distillation even in resource-constrained scenarios. These results underscore the scalability and computational efficiency of our approach, which not only accelerates large-scale dataset distillation but also substantially lowers associated computational costs.

\subsection{Cross-Architecture Generalization}
A fundamental criterion for evaluating the quality of distilled data is its ability to generalize across diverse network architectures, which significantly enhances its practical utility and adaptability in real-world applications. As illustrated in Table~\ref{tab:combined_tab_efficiency_cross} (Right), \name{}+ consistently outperforms all existing state-of-the-art methods across models of varying sizes and complexities, demonstrating superior generalization capabilities and robustness in diverse scenarios.

\subsection{Ablation Study}
\label{sec:ablation}

\noindent
\begin{wraptable}{r}{0.45\textwidth}
    \vspace{-.2in}
    \centering  
    \resizebox{\linewidth}{!}{
    \begin{tabular}{l|cc|cc}
        \toprule
                             & \multicolumn{2}{c|}{IPC=10}   & \multicolumn{2}{c}{IPC=50} \\ 
                             & \name{} & \name{}+  & \name{} & \name{}+  \\ 
        \midrule
        1 $\times$ 1 &  \textbf{48.4}  &   \textbf{50.9}    &  \textbf{60.1}   & \textbf{61.2} \\ 
        2 $\times$ 2 & 47.7 & 50.0& 59.8     &60.1   \\ 
        \bottomrule
    \end{tabular}
    }
    \vspace{-.04in}
    \caption{Comparison of student model (ResNet-18) generalization performance when trained on distilled datasets generated using $1 \times 1$ and $2 \times 2$ patch configurations during initialization and residual injection.}
    \vspace{-.3in}
    \label{tab:patch-num}
\end{wraptable}
\noindent\textbf{Impact of Patch Numbers for Initialization and Residuals.}
To assess the effect of different patch configurations during both the initialization and residual injection stages, we conduct an ablation study, as shown in Table~\ref{tab:patch-num}.
The results suggest that both $1 \times 1$ and $2 \times 2$ patch settings are effective for generating distilled data.
However, the $1 \times 1$ configuration consistently delivers the better overall performance, making it the preferred choice in practice.

\begin{table}[b]
\centering
\scriptsize
\renewcommand{\arraystretch}{1}
\resizebox{0.95\linewidth}{!}{  

    \begin{tabular}{l|cc|cc|cc|cc}
    \toprule
                         & \multicolumn{2}{c|}{\name{}}   & \multicolumn{2}{c|}{\name{}+}& \multicolumn{2}{c|}{SRe$^2$L++} & \multicolumn{2}{c}{G-VBSM} \\ 
                         & W/ MPT & W/O MPT  & W/ MPT & W/O MPT & W/ MPT & W/O MPT & W/ MPT & W/O MPT \\ 
    \midrule
    ResNet-18 (Student) & 47.7 \%    & 47.8 \%       & 50.0 \%    & 49.6 \%  & 43.1 \% & 43.1 \%  & 30.5 \% & 30.7 \%  \\ 
    Efficiency      & 0.26 ms    & 0.63  ms      & 0.58  ms     &  0.96 ms  & 0.26 ms  & 0.63 ms &2.65 ms & 4.32 ms \\ 
    Peak GPU Memory & 2.9 GB       & 5.3 GB        &  12 GB     &  23 GB  & 2.9 GB & 5.3 GB & 11.8 GB & 21.4 GB   \\ 
    \bottomrule
    \end{tabular}
}
\vspace{1em}
\caption{Comparison of model generalization performance, optimization efficiency (milliseconds per image per iteration, measured under 100 batch size and 224 as input size), and peak GPU memory usage with and without mixed precision training under ImageNet-1K IPC=10.}
\label{tab:mpt_comparison}
\vspace{-.2in}
\end{table}

\begin{table}[t]
    \centering
    \begin{subtable}{0.49\linewidth}
        \centering
        \renewcommand{\arraystretch}{1.2}
        \resizebox{\linewidth}{!}{  
        \begin{tabular}{lcc}
        \toprule
        Configuration         & Accuracy (\%) & Time Cost (s) \\ 
        \midrule
        \name{} (W/O \emph{ARC} + W/O \emph{MRO})    & 46.4          & 0.52       \\ 
        \name{} (W/O \emph{ARC} + W/ \emph{MRO})     & 46.2          & 0.47      \\ 
        \midrule
        \name{} (W/ \emph{ARC} ($\alpha=0.9$) + W/ \emph{MRO}) & 45.7          & 0.47       \\ 
        \name{} (W/ \emph{ARC} ($\alpha=0.8$) + W/ \emph{MRO}) & 46.4          & 0.47       \\ 
        \name{} (W/ \emph{ARC} ($\alpha=0.7$) + W/ \emph{MRO}) & 47.6     & 0.47       \\ 
        \name{} (W/ \emph{ARC} ($\alpha=0.6$) + W/ \emph{MRO}) & 47.3    & 0.47       \\ 
        \name{} (W/ \emph{ARC} ($\alpha=0.5$) + W/ \emph{MRO}) & \textbf{47.7}     & 0.47       \\ 
        \name{} (W/ \emph{ARC} ($\alpha=0.4$) + W/ \emph{MRO}) &   47.4       & 0.47       \\ 
        \bottomrule
        \end{tabular}
        }
    \end{subtable}
    \hfill
    \begin{subtable}{0.49\linewidth}
        \centering
        \renewcommand{\arraystretch}{1.24}
        \resizebox{\linewidth}{!}{  
        \begin{tabular}{lcc}
        \toprule
        Configuration         & Accuracy (\%) & Time Cost (s) \\ 
        \midrule
        \textbf{\name{}+} (W/O \emph{ARC} + W/O \emph{MRO})&  48.7 &     1.16 \\ 
        \textbf{\name{}+} (W/O \emph{ARC} + W/ \emph{MRO})&  48.2  &      1.09 \\ 
        \midrule
        \textbf{\name{}+} (W/ \emph{ARC} ($\alpha=0.9$) + W/ \emph{MRO})) &   48.5    &  1.09   \\ 
        \textbf{\name{}+} (W/ \emph{ARC} ($\alpha=0.8$) + W/ \emph{MRO}) &     48.0  &   1.09   \\ 
        \textbf{\name{}+} (W/ \emph{ARC} ($\alpha=0.7$) + W/ \emph{MRO}) &     48.9  &   1.09  \\ 
        \textbf{\name{}+} (W/ \emph{ARC} ($\alpha=0.6$) + W/ \emph{MRO}) &   49.3    &    1.09  \\ 
        \textbf{\name{}+} (W/ \emph{ARC} ($\alpha=0.5$) + W/ \emph{MRO}) &   \textbf{50.0}    &  1.09   \\ 
        \textbf{\name{}+} (W/ \emph{ARC} ($\alpha=0.4$) + W/ \emph{MRO}) &    49.5   &   1.09   \\ 
        \bottomrule
        \end{tabular}
        }
    \end{subtable}
    \vspace{1em}
\caption{Performance comparison of ResNet-18 as the student model trained on distilled ImageNet-1K (IPC=10) datasets generated with different merge ratios ($\alpha$), fixed $D_\text{ds}$=200 and \emph{k}=3. The efficiency is measured in seconds per image generation. \textbf{Left} presents the ablation results for single-model distillation, while \textbf{Right} shows the corresponding results for multi-model distillation.}
\label{tab:fadrm_results}
\vspace{-.05in}
\end{table}

\begin{table}[t]
    \centering
    \vspace{-0.15in}
    \scriptsize
    \begin{subtable}{0.59\linewidth}
        \centering
        \renewcommand{\arraystretch}{1}
        \begin{tabular}{c|cccccc}
            \toprule
            $k$ & 1 & 2 & 3 & 4 & 5 & 6 \\ 
            \midrule
            ImageNet-1k & 47.1 & 47.6 & \textbf{47.8} & 47.6 & 47.4 & 47.3 \\ 
            CIFAR-100   & 59.2 & 60.9 & \textbf{61.5} & 60.5 & 59.4 & 57.9 \\
            \bottomrule
        \end{tabular}
    \end{subtable}
    \hfill
    \begin{subtable}{0.4\linewidth}
        \centering
        \renewcommand{\arraystretch}{1}
        \begin{tabular}{c|cccc}
            \toprule
            $D_\text{ds}$ & 160 & 180 & 200 & 224 \\ 
            \midrule
            Post Eval (\%) & 47.2 & 47.5 & 47.7 & 47.7 \\ 
            Time Cost (s)  & 0.42 & 0.44 & 0.47 & 0.52 \\ 
            \bottomrule
        \end{tabular}
    \end{subtable}
    \vspace{1em}
    \caption{\textbf{Left} presents the ablation results for $k$ (frequency of residual connections) using \name{} with $D_\text{ds}$=200, $\alpha$=0.5, while \textbf{Right} shows the ablation results for $D_\text{ds}$ on ImageNet-1k IPC=10. Efficiency is measured as the total time required to optimize a single image under a fixed budget of 2,000 optimization iterations using \name{} with $\alpha$=0.5, $k$=3.}
\label{tab:combined_ablation_k_alpha}
    \vspace{-.2in}
\end{table}

\noindent\textbf{Impact of Mixed Precision Training (MPT).} Our ablation study in Table~\ref{tab:mpt_comparison} shows that MPT preserves distilled dataset quality while significantly reducing peak memory usage and improving optimization efficiency, making it an effective strategy for accelerating distillation.

\noindent\textbf{Impact of Components in \name{}.} To evaluate the impact of individual components (\emph{MRO} and \emph{ARC}) in our framework, we conduct a comprehensive ablation study. As demonstrated in Table~\ref{tab:fadrm_results}, the results reveal that the initial integration of \emph{MRO} results in a performance decline compared to the baseline (W/O \emph{ARC} and W/O \emph{MRO}). This decline is primarily attributed to the loss of crucial details during the resampling process. However, by incorporating \emph{ARC} and reducing the merge ratio $\alpha$ (thereby assigning higher priority to the original patches during merging), the performance is significantly enhanced compared to the baseline, which struggles with the issue of information vanishing. The optimal performance is achieved with a merge ratio of 0.5 for both \name{} and \name{}+, indicating that an equal combination of the original patch and the intermediate optimized input produces the most favorable outcomes. Crucially, the results confirm that \emph{ARC} effectively mitigates information vanishing and compensates for missing details during the resampling process, thereby facilitating the deployment of a fast yet highly effective framework.

\begin{wrapfigure}{r}{0.5\textwidth}
    \centering
    \vspace{-0.19in}
\includegraphics[width=0.99\linewidth]{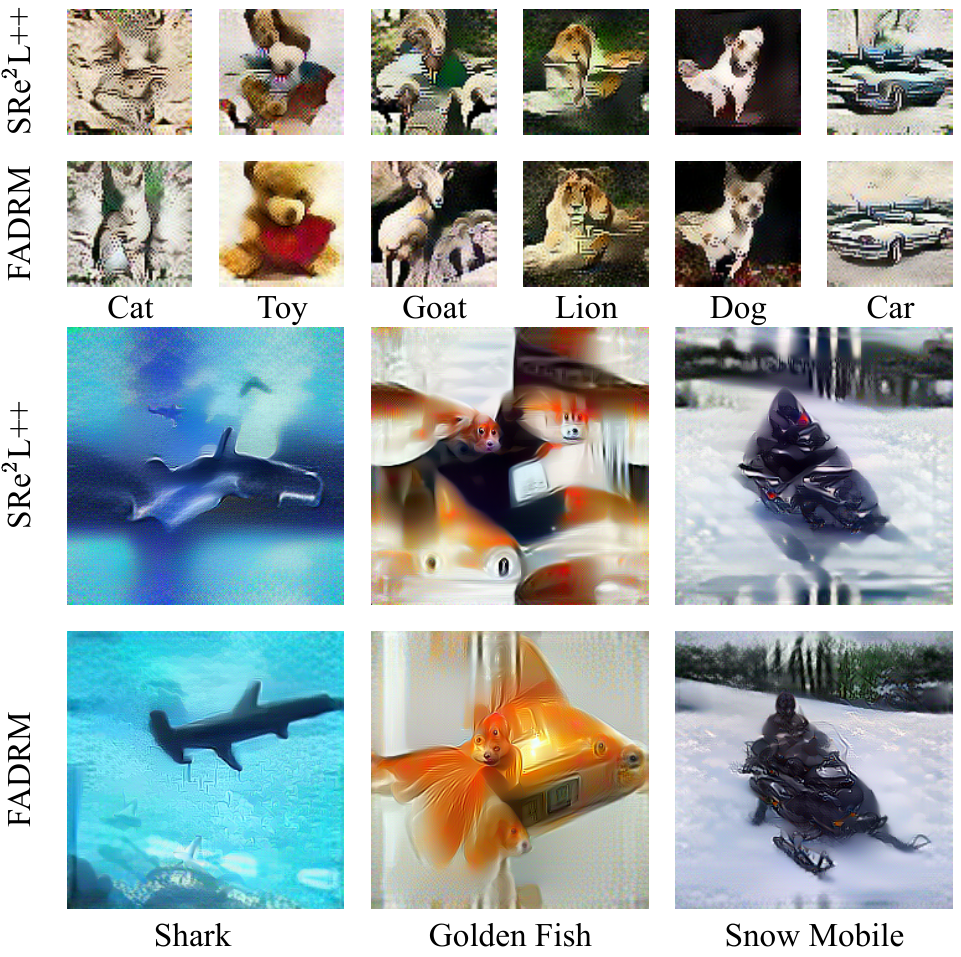}
    \caption{Visualization of dataset distilled by \name{} and SRe$^2$L++ on Tiny-ImageNet (top two rows) and ImageNet-1k (bottom two rows).
    }
    \vspace{-0.2in}
    \label{fig:DD_visualization}
\end{wrapfigure}
\noindent\textbf{Impact of Downsampled Input Size in \emph{MRO}.} To determine the optimal downsampled input size ($D_\text{ds}$) for \emph{MRO}, we conduct an ablation study, as presented in Table~\ref{tab:combined_ablation_k_alpha} (Right). Our results demonstrate that $D_\text{ds}$=200 achieves the most optimal performance. Notably, using other sizes leads to a degradation in the quality of the distilled dataset compared to optimizing with the original input size of 224.

\noindent\textbf{Impact of varying $k$.} To investigate the impact of $k$ on the quality of the distilled dataset, we conduct an ablation study as presented in Table~\ref{tab:combined_ablation_k_alpha} (Left). The results demonstrate that $k=3$ yields the optimal configuration. Notably, we observe a positive correlation between the $k$ and post-evaluation performance when increasing $k$ from one to three. However, beyond $k=3$, performance decreases as $k$ increases, indicating that excessive residual connections can introduce redundant local details over global structures, which can lead to suboptimal results.

\subsection{Distilled Image Visualization}

Fig.~\ref{fig:DD_visualization} compares distilled data from \name{} and SRe$^2$L++~\cite{cui2025dataset}, both using ResNet-18 with identical initial patch images, differing only in \name{}'s incorporation of residual connections. As shown, \name{} effectively preserves the critical features of the original patches and retains significantly more details than SRe$^2$L++. This highlights the advantage of residual connections in enhancing information density and improving the quality of distilled data. 

\subsection{Application: Continual Learning}
Leveraging continual learning to verify the effectiveness of distilled dataset generalization has been widely used in prior work~\cite{yin2024squeeze,GBVSM_2024,DBLP:conf/wacv/ZhaoB23}. Following these protocols and utilizing the class-incremental learning framework as in DM~\cite{DBLP:conf/wacv/ZhaoB23}, we conduct an evaluation on Tiny-ImageNet IPC=50 using a 5-step and 10-step incremental, as shown in Fig.~\ref{fig:cl}. The results indicate that \name{} consistently surpasses RDED, demonstrating its effectiveness. 

\begin{figure}[htbp]
    \centering
    \begin{subfigure}[b]{0.45\textwidth}
        \includegraphics[width=\textwidth]{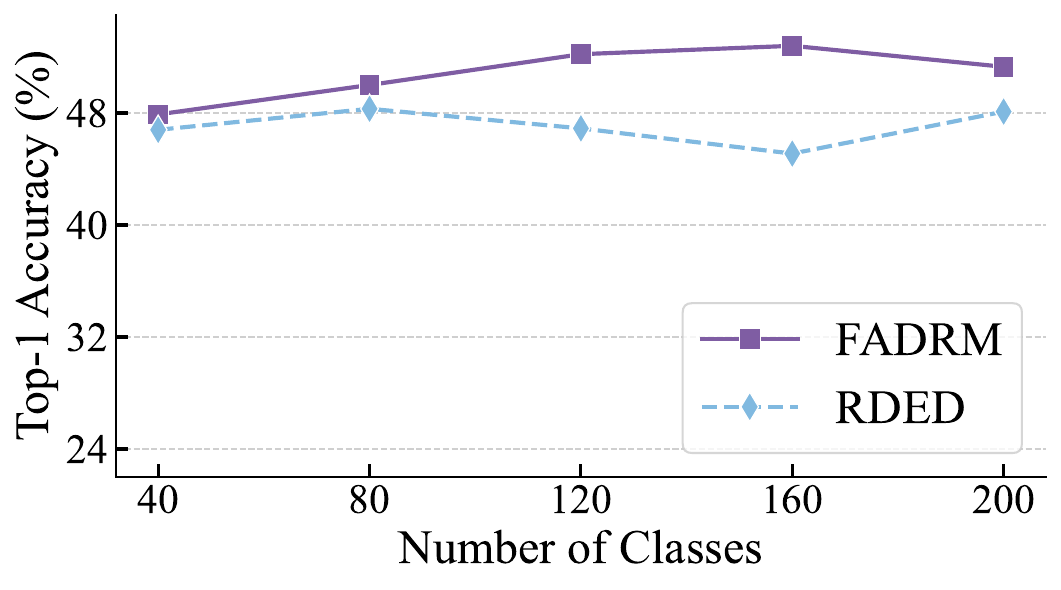}
    \end{subfigure}
    \hfill
    \begin{subfigure}[b]{0.45\textwidth}
        \includegraphics[width=\textwidth]{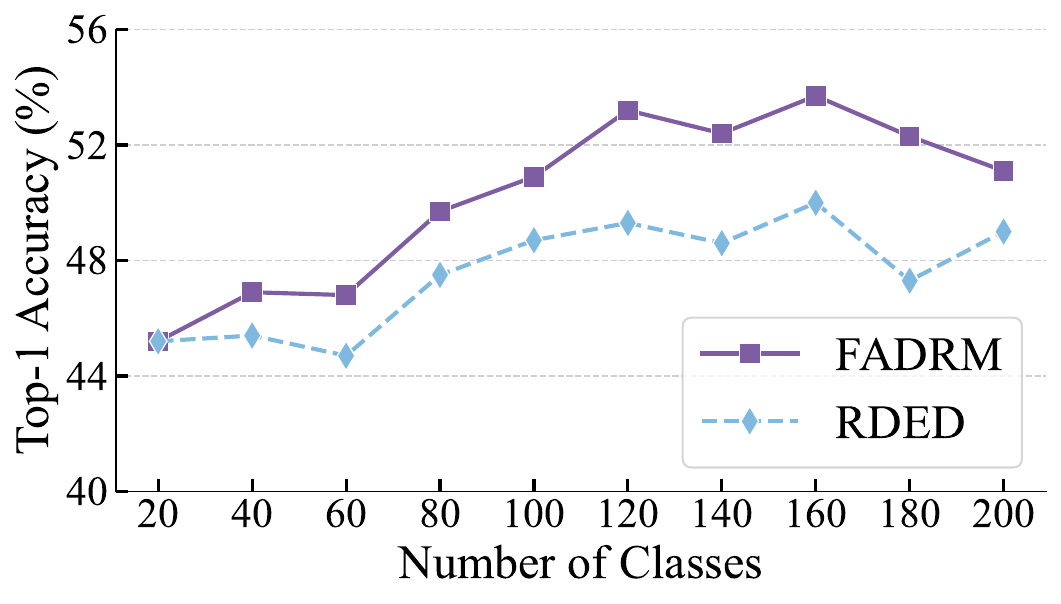}
    \end{subfigure}
    \caption{Five-step and Ten-step class-incremental learning on Tiny-ImageNet with IPC=50.}
    \label{fig:cl}
\end{figure}
\section{Conclusion}
\label{sec:conclusion} 
We proposed \name{}, a novel framework for dataset distillation designed to generate high-quality distilled datasets with significantly reduced computational overhead. Our work identifies and addresses the critical challenge of vanishing information, a fundamental limitation in {\em Uni-Level Framework} that heavily undermines the information density of distilled datasets. To address this, we introduce {\em data-level residual connections}, a novel mechanism that balances the operations of preserving critical original features and integrating new information, enriching the distilled dataset with both original and new condensed features and increasing its overall information density. Furthermore, by integrating parameter mixed precision training and input multi-resolution optimization, our framework achieves significant reductions in both Peak GPU memory consumption and training time. Extensive experiments demonstrate that \name{} outperforms existing state-of-the-art methods in both efficiency and accuracy across multiple benchmark datasets. For future work, we aim to extend the idea of {\em data-level residual connections} to broader modalities and applications of dataset distillation tasks.

\bibliographystyle{plainnat}
\bibliography{reference}

\clearpage
\appendix
\newpage

\begin{center}
	\LARGE \bf {Appendix of FADRM}
\end{center}
\etocdepthtag.toc{mtappendix}
\etocsettagdepth{mtchapter}{none}
\etocsettagdepth{mtappendix}{section}
\tableofcontents
\newpage

\section{Theoretical Derivation} 
\label{sec:Theoretical Derivation}

\subsection{Preliminary}

\begin{lemma}[Data Processing Inequality~\cite{beaudry2012intuitiveproofdataprocessing}]
\label{lemma:data-processing}
Let $X \to Y \to Z$ form a Markov chain. Then the mutual information between $X$ and $Z$ is upper bounded by that between $X$ and $Y$:
\begin{equation}
    I(X; Z) \le I(X; Y).
\end{equation}

In particular, no post-processing of $Y$ can increase the information that $Y$ contains about $X$.
\end{lemma}

\begin{theorem}[Temperature-scaled KL divergence is bounded and Lipschitz‐continuous]
\label{thm:kl_bounded_full}
Fix integers $k\ge 2$ and a constant $C>0$.  
For any temperature $T>0$ let  
\begin{equation}
z=(z_1,\dots ,z_k)\in[-C,C]^k,
\qquad
q^{(T)}_i
=\frac{\exp\!\bigl(z_i/T\bigr)}{\displaystyle\sum_{j=1}^k \exp\!\bigl(z_j/T\bigr)}
\quad(i=1,\dots ,k).
\end{equation}
Let $p=(p_1,\dots ,p_k)\in\Delta_k$ be an arbitrary target probability vector
(e.g.\ it may come from another soft-max with its own temperature).
Define the loss
\begin{equation}
\ell\bigl(p,q^{(T)}(z)\bigr)
:=\KL\!\bigl(p\;\|\;q^{(T)}(z)\bigr)
=\sum_{i=1}^k p_i\log\frac{p_i}{q^{(T)}_i}.
\end{equation}
Then the following hold:
\begin{enumerate}
\item (\emph{Bounded range}) For every admissible pair $(p,z)$,
\begin{equation}
      0\;\le\;
      \ell\bigl(p,q^{(T)}(z)\bigr)
      \;\le\;
      B,
      \qquad
      B:=\log k+\frac{2C}{T}.
\end{equation}
\item (\emph{$\ell_\infty$-Lipschitz continuity in logits})  
      The map $z\mapsto\ell\bigl(p,q^{(T)}(z)\bigr)$ is
      $L$-Lipschitz w.r.t.\ the $\ell_\infty$ norm with
      $L=\frac{1}{T}$. Consequently, it is $\sqrt{k}/T$-Lipschitz w.r.t.\ the Euclidean norm.
\end{enumerate}
\end{theorem}

\begin{proof}[proof of Theorem~\ref{thm:kl_bounded_full}]
\textbf{(i) Boundedness.}
Write
\begin{equation}
\KL\!\bigl(p\;\|\;q^{(T)}\bigr)
=\sum_{i=1}^k p_i\log p_i-\sum_{i=1}^k p_i\log q^{(T)}_i .
\end{equation}
Since $x\mapsto x\log x$ is non-positive on $[0,1]$, the first term is
at most $0$, so
\begin{equation}
\KL\!\bigl(p\;\|\;q^{(T)}\bigr)
\le -\sum_{i=1}^k p_i\log q^{(T)}_i .
\end{equation}
For the soft-max,
$\log q^{(T)}_i = z_i/T-\log Z$, where
$Z:=\sum_{j=1}^k\exp(z_j/T)$.
Hence
\begin{equation}
-\sum_{i=1}^k p_i\log q^{(T)}_i
= -\frac1T\sum_{i=1}^k p_i z_i +\log Z .
\end{equation}
Because each $z_i\in[-C,C]$ and $\sum_i p_i=1$,
\begin{equation}
-\frac1T\sum_{i} p_i z_i \;\le\; \frac{C}{T}.
\end{equation}
Moreover, $z_i\le C$ implies $Z\le k\exp(C/T)$ and thus
$\log Z\le\log k + \tfrac{C}{T}$.
Combining the two parts yields the desired upper bound
$\log k + 2C/T$.  
Non-negativity of KL divergence gives the lower bound~$0$.

\medskip
\noindent\textbf{(ii) Lipschitz continuity.}
Differentiate $\ell$ w.r.t.\ $z_i$:
\begin{equation}
\partial_{z_i}\ell\bigl(p,q^{(T)}(z)\bigr)
= -\frac{p_i-q^{(T)}_i}{T}.
\end{equation}
Because $|p_i-q^{(T)}_i|\le 1$, we have
$|\partial_{z_i}\ell|\le 1/T$ for every coordinate.
Thus $\|\nabla_z\ell\|_{\infty}\le 1/T$,
and by the mean-value theorem,
\begin{equation}
\bigl|\ell(p,q^{(T)}(z))-\ell(p,q^{(T)}(z'))\bigr|
\le \frac1T\|z-z'\|_{\infty},
\quad\forall z,z'\in[-C,C]^k,
\end{equation}
so $L=1/T$ in the $\ell_\infty$ norm.
Since $\|v\|_2\le\sqrt{k}\|v\|_\infty$, the Euclidean Lipschitz
constant is at most $\sqrt{k}/T$.
\end{proof}

\begin{lemma}[Generalization Bound via Rademacher Complexity~\cite{bartlett2002rademacher}]
\label{lemma:bartlett}
Let \( \mathcal{H} \) be a class of functions mapping \( \mathcal{X} \to [0, B] \), and let \( S = \{x_1, \dots, x_n\} \) be an i.i.d. sample from distribution \( \mathcal{D} \). Then, for any \( \delta > 0 \), with probability at least \( 1 - \delta \), the following inequality holds for all \( h \in \mathcal{H} \):
\begin{equation}
\mathbb{E}_{x \sim \mathcal{D}}[h(x)]  \leq \frac{1}{n} \sum_{i=1}^n h(x_i) + 2 \MyRad(\mathcal{H}) + B \sqrt{\frac{\log(1/\delta)}{2n}}
\end{equation}
\end{lemma}

\begin{lemma}[Empirical Risk Proximity]
\label{lemma:empirical-risk-alignment}
Let \( \tilde{x}_i := \alpha \tilde{x}_i^{\mathrm{res}} + (1 - \alpha) x_i \) with \( \alpha \in (0, 1) \), and let the corresponding datasets be
$
\tilde{\mathcal{C}}^{\mathrm{res}} := \{ (\tilde{x}_i^{\mathrm{res}}, y_i) \}_{i=1}^n, \quad
\tilde{\mathcal{C}}_{\mathrm{FADRM}} := \{ (\tilde{x}_i, y_i) \}_{i=1}^n.
$
Then for any model \( h \in \mathcal{H} \), the empirical risk difference is bounded by a negligible value:
\begin{equation}
\left| \widehat{\mathcal{L}}_{\mathrm{res}}(h) - \widehat{\mathcal{L}}_{\mathrm{FADRM}}(h) \right|
\leq L_\ell L_h (1 - \alpha) \cdot \Delta_1,
\quad \text{where } \Delta_1 := \frac{1}{n} \sum_{i=1}^n \| \tilde{x}_i^{\mathrm{res}} - x_i \|.
\end{equation}
\end{lemma}

\begin{proof}[Proof of Lemma~\ref{lemma:empirical-risk-alignment}]
We begin by computing the pointwise difference in the loss:
\begin{equation}
\left| \ell(h(\tilde{x}_i^{\mathrm{res}}), y_i) - \ell(h(\tilde{x}_i), y_i) \right|.
\end{equation}
Since \( \ell \) is \( L_\ell \)-Lipschitz in the model output, and \( h \) is \( L_h \)-Lipschitz in the input, we have:
\begin{equation}
\left| \ell(h(\tilde{x}_i^{\mathrm{res}}), y_i) - \ell(h(\tilde{x}_i), y_i) \right|
\leq L_\ell \cdot | h(\tilde{x}_i^{\mathrm{res}}) - h(\tilde{x}_i) |
\leq L_\ell L_h \cdot \| \tilde{x}_i^{\mathrm{res}} - \tilde{x}_i \|.
\end{equation}
Note that:
\begin{equation}
\tilde{x}_i = \alpha \tilde{x}_i^{\mathrm{res}} + (1 - \alpha) x_i
\quad \Rightarrow \quad
\tilde{x}_i^{\mathrm{res}} - \tilde{x}_i = (1 - \alpha)(\tilde{x}_i^{\mathrm{res}} - x_i),
\end{equation}
so:
\begin{equation}
\| \tilde{x}_i^{\mathrm{res}} - \tilde{x}_i \| = (1 - \alpha) \| \tilde{x}_i^{\mathrm{res}} - x_i \|.
\end{equation}
Therefore,
\begin{equation}
\left| \ell(h(\tilde{x}_i^{\mathrm{res}}), y_i) - \ell(h(\tilde{x}_i), y_i) \right|
\leq L_\ell L_h (1 - \alpha) \| \tilde{x}_i^{\mathrm{res}} - x_i \|.
\end{equation}
Averaging over \( n \) samples:
\begin{equation}
\left| \widehat{\mathcal{L}}_{\mathrm{res}}(h) - \widehat{\mathcal{L}}_{\mathrm{FADRM}}(h) \right|
\leq \frac{1}{n} \sum_{i=1}^n L_\ell L_h (1 - \alpha) \| \tilde{x}_i^{\mathrm{res}} - x_i \|
= L_\ell L_h (1 - \alpha) \cdot \Delta_1.
\end{equation}
\end{proof}

\begin{corollary}[Lipschitz Convex Combination Bound]
\label{corollary:lipschitz-combination}
Let \( h: \R^d \to \R \) be an \( L \)-Lipschitz function. For any \( x, y \in \R^d \) and \( \alpha \in (0,1) \), define \( z = \alpha x + (1 - \alpha)y \). Then:
\begin{equation}
\left| h(z) - \left( \alpha h(x) + (1 - \alpha) h(y) \right) \right| \leq L \alpha(1 - \alpha) \|x - y\|
\end{equation}
In particular, this implies:
\begin{equation}
h(z) \leq \alpha h(x) + (1 - \alpha) h(y) + L \alpha(1 - \alpha) \|x - y\|
\end{equation}
\begin{equation}
h(z) \geq \alpha h(x) + (1 - \alpha) h(y) - L \alpha(1 - \alpha) \|x - y\|
\end{equation}
\end{corollary}

\subsection{Bounded Information in BN-Aligned Synthetic Data}
\label{proof:Bounded Information in BN-Aligned Synthetic Data}

\begin{proof} [Proof of Theorem~\ref{theorem:Bounded Information in BN-Aligned Synthetic Data}]
Let $\mathcal{O}$ denote the original dataset. From it, a pretrained model $f_\theta$ is derived, which includes BatchNorm statistics $\{\mu_l, \sigma_l^2\}$. Each synthetic image $\tilde{x}_j$ in the distilled dataset $\mathcal{C}$ is generated by minimizing an objective function depending only on $f_\theta$ and a fixed label $\tilde{y}_j$.

We assume that each $\tilde{x}_j$ is generated independently given $f_\theta$, and that $f_\theta$ is a deterministic function of $\mathcal{O}$. Then, for each sample $(\tilde{x}_j, \tilde{y}_j)$, we have the Markov chain:
\begin{equation}
\mathcal{O} \to f_\theta \to \tilde{x}_j,
\end{equation}
By applying Lemma~\ref{lemma:data-processing}, we get:
\begin{equation}
I(\tilde{x}_j; \mathcal{O}) \le I(f_\theta; \mathcal{O}) = H(f_\theta),
\end{equation}
Now, by the chain rule of mutual information:
\begin{equation}
I(\mathcal{C}; \mathcal{O}) = I(\{\tilde{x}_j, \tilde{y}_j\}_{j=1}^{|\mathcal{C}|}; \mathcal{O}) \le \sum_{j=1}^{|\mathcal{C}|} I(\tilde{x}_j; \mathcal{O}) \le |\mathcal{C}| \cdot H(f_\theta),
\end{equation}
where we used the fact that $\tilde{y}_j$ is fixed and independent of $\mathcal{O}$ and the independence assumption across samples. Thus, the total information that the synthetic dataset $\mathcal{C}$ can retain about the original dataset $\mathcal{O}$ is bounded by the product of its size and the entropy of the model $f_\theta$.
\end{proof}

\subsection{\emph{ARC} improves the robustness of the distilled images}
\label{proof:ARC improves robustness}

\begin{proof}[Proof of Theorem~\ref{theorem:ARC improves robustness}]
Let \( \tilde{x}_i^{res}\) be a perturbation generated via distribution (running statistics) matching and prediction (cross entropy) matching, and let \( x_i  \) be a real image from the original dataset. 

Define the residual-injected sample $\tilde{x}_i$ as:
\begin{equation}
\tilde{x}_i := \alpha \tilde{x}_i^{res} + (1 - \alpha) x_i, \quad \alpha \in (0,1)
\end{equation}
Define the datasets:
\begin{itemize}
    \item \( \tilde{\mathcal{C}}^{res} = \{ \tilde{x}_i^{res}, \tilde{y}_i^{res} \}_{i=1}^n \): perturbation generated via distribution (running statistics) matching and prediction (cross entropy) matching,
    \item \( \mathcal{O} = \{ x_i, y_i \}_{i=1}^n \): selected patches from the original dataset,
    \item \( \tilde{\mathcal{C}}_{\mathrm{FADRM}} = \{ \tilde{x}_i, \tilde{y}_i \}_{i=1}^n \): residual-injected dataset.
\end{itemize}

We begin by bounding the Rademacher complexity of the residual-injected dataset \( \tilde{\mathcal{C}}_{\mathrm{FADRM}} = \{ \tilde{x}_i \}_{i=1}^n \), where \( \tilde{x}_i = \alpha \tilde{x}_i^{res} + (1 - \alpha) x_i \), using Lemma~\ref{lemma:bartlett} and Corollary~\ref{corollary:lipschitz-combination}.

From the definition:
\begin{equation}
\MyRad(\Hcal \circ \tilde{\mathcal{C}}_{\mathrm{FADRM}}) = \mathbb{E}_{\boldsymbol{\sigma}} \left[ \sup_{h \in \Hcal} \frac{1}{n} \sum_{i=1}^n \sigma_i h(\tilde{x}_i) \right]
\end{equation}
By Corollary~\ref{corollary:lipschitz-combination}, we have for each term:
\begin{equation}
h(\tilde{x}_i) \leq \alpha h(\tilde{x}_i^{res}) + (1 - \alpha) h(x_i) + \varepsilon_i,
\quad \text{where } |\varepsilon_i| \leq L_{h} \cdot \alpha(1 - \alpha) \| \tilde{x}_i^{res} - x_i \|
\end{equation}
Therefore:
\begin{equation}
\sum_{i=1}^n \sigma_i h(\tilde{x}_i) \leq \sum_{i=1}^n \sigma_i \left( \alpha h(\tilde{x}_i^{res}) + (1 - \alpha) h(x_i) \right) + \sum_{i=1}^n |\sigma_i \varepsilon_i|
\end{equation}
Using \( |\sigma_i| = 1 \), we get:
\begin{equation}
\sum_{i=1}^n |\sigma_i \varepsilon_i| \leq L_{h} \alpha(1 - \alpha) \sum_{i=1}^n \| \tilde{x}_i^{res} - x_i \| = n \cdot L_{h} \alpha(1 - \alpha) \cdot \Delta
\end{equation}
Divide by \( n \), take supremum and expectation:

\begin{equation}
\label{eq:generalization bound}
    \MyRad(\Hcal \circ \tilde{\mathcal{C}}_{\mathrm{FADRM}}) \leq \alpha \cdot \MyRad(\Hcal \circ \tilde{\mathcal{C}}^{res}) + (1 - \alpha) \cdot \MyRad(\Hcal \circ \mathcal{O}) + L_{h} \alpha(1 - \alpha) \cdot \Delta
\end{equation}

Rearrange the Inequality:
\begin{equation}
\mathfrak{R}_n(\mathcal{H} \circ \tilde{\mathcal{C}}_{\mathrm{FADRM}})
- \mathfrak{R}_n(\mathcal{H} \circ \tilde{\mathcal{C}}^{\mathrm{res}})
\leq (1 - \alpha) \left[
\mathfrak{R}_n(\mathcal{H} \circ \mathcal{O})
- \mathfrak{R}_n(\mathcal{H} \circ \tilde{\mathcal{C}}^{\mathrm{res}})
\right]
+ L_{h} \alpha(1 - \alpha) \cdot \Delta
\end{equation}

Multiply $2B$ on both sides and add a negligible positive value $\epsilon$ to the LHS:
\begin{equation}
\begin{aligned}
    2B \cdot \big[ 
        \mathfrak{R}_n(\mathcal{H} \circ \tilde{\mathcal{C}}_{\mathrm{FADRM}})
        - \mathfrak{R}_n(\mathcal{H} \circ \tilde{\mathcal{C}}^{\mathrm{res}})
    \big]
    < \;&
    2B(1 - \alpha) \cdot 
    \big[ 
        \mathfrak{R}_n(\mathcal{H} \circ \mathcal{O}) 
        - \mathfrak{R}_n(\mathcal{H} \circ \tilde{\mathcal{C}}^{\mathrm{res}})
    \big] \\
    &+ 2B L_{h} \alpha(1 - \alpha) \cdot \Delta + \epsilon
\end{aligned}
\end{equation}

As validated in Theorem~\ref{thm:kl_bounded_full}, when $T >$ 0, KL-divergence becomes a bounded \( B \)-range loss, which we then apply Lemma~\ref{lemma:bartlett} to formulate generalization error:
\begin{equation}
\mathcal{L}_{\mathrm{gen}}(h) \leq \widehat{\mathcal{L}}(h) + 2B \cdot \MyRad(\Hcal \circ S)
\end{equation}

Apply to both models:
\begin{align}
\mathcal{L}_{\mathrm{gen}}(h_{\mathrm{res}}) &\leq \widehat{\mathcal{L}}_{\mathrm{res}} + 2B \cdot \MyRad(\Hcal \circ \tilde{\mathcal{C}}^{res}) \label{eq:gen_rs} \\
\mathcal{L}_{\mathrm{gen}}(h_{\mathrm{FADRM}}) &\leq \widehat{\mathcal{L}}_{\mathrm{FADRM}} + 2B \cdot \MyRad(\Hcal \circ \tilde{\mathcal{C}}_{\mathrm{FADRM}}) \label{eq:gen_res}
\end{align}

Recall the lower bound for the difference of two ERMs established in Lemma~\ref{lemma:empirical-risk-alignment}, we then have:
\begin{equation}
     \widehat{\mathcal{L}}_{\mathrm{res}}(h) - \widehat{\mathcal{L}}_{\mathrm{FADRM}}(h)
\geq - L_\ell L_h (1 - \alpha) \cdot \Delta,
\quad \text{where } \Delta := \frac{1}{n} \sum_{i=1}^n \| \tilde{x}_i^{\mathrm{res}} - x_i \|.
\end{equation}

Given the assumption~\eqref{eq:residual-upper-bound}, we can then derive:
\begin{equation}
\label{eq:intermediate res}
    -L_\ell L_h (1 - \alpha) \cdot \Delta > 2B \Bigg\{ (1-\alpha) \left[ \MyRad(\Hcal \circ \mathcal{O}) - \MyRad(\Hcal \circ \tilde{\mathcal{C}}_{\mathrm{res})}\right] + L+{h} \alpha (1-\alpha)\cdot \Delta \Bigg\} +\epsilon
\end{equation}

where the RHS in Equation~\eqref{eq:intermediate res} is the upper bound for the difference in Rademacher Complexity, we then derive the following inequality:
\begin{equation}
    \widehat{\mathcal{L}}_{\mathrm{res}} - \widehat{\mathcal{L}}_{\mathrm{FADRM}} > 2B \cdot \left[  \MyRad(\Hcal \circ \tilde{\mathcal{C}}_{\mathrm{FADRM}}) - \MyRad(\Hcal \circ \tilde{\mathcal{C}}^{res})\right]
\end{equation}

which shows:
\begin{equation}
    \widehat{\mathcal{L}}_{\mathrm{res}} + 2B \cdot \MyRad(\Hcal \circ \tilde{\mathcal{C}}^{res}) > \widehat{\mathcal{L}}_{\mathrm{FADRM}} + 2B \cdot \MyRad(\Hcal \circ \tilde{\mathcal{C}}_{\mathrm{FADRM}})
\end{equation}

\end{proof}

\section{Optimization Details}
\label{sec:Optimization_Details}
Formally, the optimization process adheres to the principle of aligning the synthesized data with both the predictive behavior and the statistical distribution captured by a pretrained model $f_{\theta}$. Specifically, given a synthesized image $\tilde{x}_t$ at iteration $t$, the optimization objective is defined as:
\begin{align}
    \operatorname*{arg\,min}_{\tilde{x}_t} \quad \mathcal{L}(f_{\theta}(\tilde{x}_t), \tilde{\mathbf{y}}) + \mathcal{D}_{\text{global}}(\tilde{x}_t),
\end{align}
where $l(f_{\theta}(\tilde{x}_t), \tilde{\mathbf{y}})$ enforces consistency with the target predictions, while $\mathcal{D}_{\text{global}}(\tilde{x}_t)$ ensures alignment with the statistical distribution. Importantly, the parameters of $f_{\theta}$ remain fixed throughout the optimization, and only $\tilde{x}_t$ is updated.

The prediction alignment term is formulated as the cross-entropy loss computed over the synthesized batch:
\begin{equation}
    \mathcal{L}(f_{\theta}(\tilde{x}_t), \tilde{\mathbf{y}}) = - \frac{1}{N} \sum_{n=1}^{N} \sum_{i=1}^{C} \tilde{\mathbf{y}}_{n,i} \log f_{\theta}(\tilde{x}_t)_{n,i},
\end{equation}
where $N$ denotes the batch size, and $C$ represents the total number of classes. The alignment to the distribution in pretrained model is calculated as follows:
\begin{align*}
    \mathcal{D}_{\text{global}}(\tilde{x}_t) &= \sum_{l} \left\| \mu_{l} (\tilde{x}_t) - \mathbb{E}[\mu_{l}|\mathcal{O}] \right\|_{2} \\
    &\quad + \sum_{l} \left\| \sigma_{l}^2 (\tilde{x}_t) - \mathbb{E}[\sigma_{l}^2|\mathcal{O}] \right\|_{2} \\
    &= \sum_{l} \left\| \mu_{l} (\tilde{x}_t) - \mathbf{BN}_{l}^{\text{RM}} \right\|_{2} \\
    &\quad + \sum_{l} \left\| \sigma_{l}^2 (\tilde{x}_t) - \mathbf{BN}_{l}^{\text{RV}} \right\|_{2},
\end{align*}
where $\mathcal{O}$ denotes the original dataset, and $l$ indexes the layers of the model. The terms $\mathbf{BN}_{l}^{\text{RM}}$ and $\mathbf{BN}_{l}^{\text{RV}}$ correspond to the running mean and running variance of the Batch Normalization (BN) statistics at layer $l$. By minimizing $\mathcal{D}_{\text{global}}(\tilde{x}_t)$, the synthesized data is encouraged to exhibit statistical characteristics consistent with the original dataset, thereby preserving global information.

\section{Resampling via Bilinear Interpolation}
\label{Resampling}
Given an original image \( I : \mathbb{Z}^2 \to \mathbb{R}^C \) defined on discrete pixel coordinates, the continuous extension \( \tilde{I} : \mathbb{R}^2 \to \mathbb{R}^C \) at non-integer location \( (i', j') \in \mathbb{R}^2 \) is computed via bilinear interpolation as follows:
\begin{equation}
    \tilde{I}(i', j') = \sum_{m=0}^{1} \sum_{n=0}^{1} w_{m,n} \cdot I(i + m, j + n),
    \label{eq:bilinear_general}
\end{equation}
where \( i = \lfloor i' \rfloor \), \( j = \lfloor j' \rfloor \), \( \alpha = i' - i \in [0,1) \), \( \beta = j' - j \in [0,1) \), and the interpolation weights are defined by:
\begin{equation}
    w_{m,n} = (1 - m + (-1)^m \alpha)(1 - n + (-1)^n \beta).
\end{equation}

\noindent Explicitly, Equation~\eqref{eq:bilinear_general} expands to:
\begin{equation}
    \begin{aligned}
        \tilde{I}(i', j') &= (1 - \alpha)(1 - \beta) \cdot I(i, j) + \alpha(1 - \beta) \cdot I(i+1, j) \\
        &\quad + (1 - \alpha)\beta \cdot I(i, j+1) + \alpha\beta \cdot I(i+1, j+1),
    \end{aligned}
    \label{eq:bilinear_expanded}
\end{equation}

\noindent This interpolation scheme can be viewed as a separable approximation to the continuous image function, with weights derived from tensor-product linear basis functions over the unit square. It preserves differentiability with respect to the fractional coordinates \( (i', j') \), making it particularly amenable to gradient-based optimization frameworks.

\section{Limitations}
\label{sec:limitations}
While FADRM offers substantial improvements in computational efficiency and performance for dataset distillation, it also introduces several limitations. First, the method relies on the assumption that residual signals between synthetic and real data capture critical learning dynamics, which may not generalize across domains with highly abstract or non-visual modalities such as natural language or time-series data. Second, the use of distilled datasets can inadvertently reinforce biases present in the original data if not carefully audited, potentially leading to fairness concerns in downstream applications. From a broader societal perspective, while FADRM reduces the computation and resource demands of training large models, thereby contributing positively to sustainability, it may also facilitate the deployment of powerful models in low-resource or surveillance scenarios without adequate ethical oversight. Thus, responsible deployment and continued research into bias mitigation and cross-domain generalization are essential to ensure the safe and equitable application of FADRM.

\section{Experimental Setup}
\label{sec:experimental_setup}
Our method strictly follows the training configuration established in EDC to ensure a fair and consistent comparison across all evaluated approaches. Additionally, we re-run RDED and CV-DD under the same configuration and report the highest performance obtained between their original setup and the EDC configuration. This methodology guarantees a rigorous and equitable evaluation by accounting for potential variations in training dynamics across different settings.  

To establish an upper bound on performance across different backbone architectures (representing the results achieved when training models on the full original dataset) we adopt the hyperparameters specified in Table~\ref{tab:validate_whole}. These hyperparameters are carefully chosen to ensure full model convergence while effectively mitigating the risk of overfitting, thereby providing a reliable reference for evaluating the performance of distilled datasets.

\begin{table}[htbp!]
\centering
\footnotesize
\renewcommand{\arraystretch}{1}  
\resizebox{0.67\columnwidth}{!}{
\begin{tabular}{p{3.5cm} p{4.5cm}}
\toprule
\multicolumn{2}{l}{\textbf{Hyperparameters for Training the Original Dataset}} \\
\midrule
Optimizer      & SGD                        \\ 
Learning Rate & 0.1 \\
Weight Decay & 1e-4 \\
Momentum & 0.9 \\
Batch Size & 128\\
Loss Function  & Cross-Entropy               \\ 
Epochs         & 300                       \\
Augmentation   & RandomResizedCrop, \newline Horizontal Flip, CutMix \\
\bottomrule
\end{tabular}
}
\vspace{1em}
\caption{Hyperparameters for Training the Original Dataset.}
\label{tab:validate_whole}
\end{table}

\section{Hyper-Parameters Setting}
\label{sec:hyper-parameters}
In summary, the synthesis of distilled data follows consistent hyperparameter configurations, as outlined in Table~\ref{tab:recover}. Variations in hyperparameters are introduced exclusively during two phases: (1) the model Pre-training phase. and (2) the post-evaluation phase. These adjustments are carefully tailored based on the scale of the models and the specific characteristics of the datasets used. During the post-evaluation phase, we evaluate a total of four hyperparameter combinations, as detailed in Table~\ref{tab:settings}. Among these, the parameter $\eta$ plays a critical role in controlling the decay rate of the learning rate, as defined by the cosine learning rate schedule in Equation~\ref{eq:cosine_lr}. Specifically, a larger value of $\eta$ results in a slower decay rate, thereby preserving a higher learning rate for a longer duration during training.

\begin{equation}
\textit{Learning Rate} = 0.5 \times \left( 1 + \cos\left( \pi \frac{\textit{step}}{\textit{epochs} \times \eta} \right) \right)
\label{eq:cosine_lr}
\end{equation}

\begin{table}[ht]
\centering
\footnotesize
\renewcommand{\arraystretch}{1}  
\begin{tabularx}{0.7\linewidth}{X X} 
\toprule
\textbf{Hyperparameter} & \textbf{Value} \\
\midrule
Optimizer & Adam \\
Learning rate & 0.25 \\
Beta & (0.5, 0.9) \\
Epsilon & $1 \times 10^{-8}$ \\
Batch Size & 100 or 10 (if $C < 100$) \\
Iterations Budgets ($\mathcal{B}$) & 2,000 \\
Merge Ratio ($\alpha$) & 0.5 \\
Number of \emph{ARC} ($k$) & 3 \\
Downsampled Size ($D_{\text{ds}}$) & 200 (ImageNet-1k and Its subsets), Original Input Size (CIFAR-100, Tiny-ImageNet)\\ 
\name{} Model ($R$) & ResNet18 \\ 
\name{}+ Model ($R$) & ResNet18 DenseNet121 ShuffleNetV2 MobileNetV2 \\
Scheduler & Cosine Annealing \\
Augmentation & RandomResizedCrop, Horizontal Flip \\
\bottomrule
\end{tabularx}
\vspace{1em}
\caption{Hyperparameters for generating the distilled datasets.} 
\label{tab:recover}
\end{table}

\begin{table}[ht]
\centering
\footnotesize\
\renewcommand{\arraystretch}{1}  
\resizebox{0.35\columnwidth}{!}{ 
\begin{tabular}{lcc} 
\toprule
\textbf{Setting} & \textbf{Learning Rate} & \textbf{$\eta$} \\ 
\midrule
S1 & 0.001 & 1 \\ 
S2 & 0.001 & 2 \\ 
S3 & 0.0005 & 1 \\ 
S4 & 0.0005 & 2 \\ 
\bottomrule
\end{tabular}
}
\vspace{1em}
\caption{Hyperparameter settings with learning rate and $\eta$.} 
\label{tab:settings}
\end{table}

\subsection{CIFAR-100}
This subsection outlines the hyperparameter configurations employed in the CIFAR-100 experiments, providing the necessary details to ensure reproducibility in future research.

\noindent \textbf{Pre-training phase.} Table~\ref{tab:squeeze_cifar100} provides a comprehensive summary of the hyperparameters employed for training the models on the original CIFAR-100 dataset for generating the distilled dataset.

\begin{table}[ht]
\centering
\renewcommand{\arraystretch}{1} 
\begin{tabular}{cc} 
\toprule
\multicolumn{2}{c}{\textbf{Hyperparameters for Model Pre-training}} \\
\midrule
Optimizer & SGD\\
Learning Rate & 0.1 \\
Weight Decay & 1e-4 \\
Momentum & 0.9 \\
Batch Size & 128\\
Epoch & 50 \\
Scheduler & Cosine Annealing \\
Augmentation & RandomCrop, Horizontal Flip \\
Loss Function & Cross-Entropy \\
\bottomrule
\end{tabular}
\vspace{1em}
\caption{Hyperparameters for CIFAR-100 Pre-trained Models.} 
\label{tab:squeeze_cifar100}
\end{table}

\noindent \textbf{Evaluation Phase.} Table~\ref{tab:validate_cifar100} outlines the hyperparameter configurations employed for the post-evaluation phase on the Distilled CIFAR-100 dataset.

\begin{table}[htbp!]
\centering
\renewcommand{\arraystretch}{1} 
\resizebox{0.6\columnwidth}{!}{
\begin{tabular}{p{3.5cm} p{4.5cm}}
\toprule
\multicolumn{2}{c}{\textbf{Hyperparameters for Post-Eval on R18, R50 and R101}} \\
\midrule
Optimizer      & Adamw                         \\ 
S1  &  IPC1 (R50), IPC50 (R18,R50)              \\ 
S2  &  IPC10 (R18, R50)             \\ 
S3  &  IPC1 (R101), IPC10 (R101), IPC50 (R101)          \\ 
S4  &  IPC1 (R18)         \\ 
Soft Label Generation & BSSL \\
Loss Function  & KL-Divergence                \\ 
Batch Size     & 16                          \\ 
Epochs         & 1000                       \\
Augmentation   & RandomResizedCrop, \newline Horizontal Flip, CutMix \\
\bottomrule
\end{tabular}
}
\vspace{1em}
\caption{Hyperparameters for post-evaluation task on ResNet18, ResNet50 and ResNet101 for CIFAR-100.}
\label{tab:validate_cifar100}
\end{table}

\subsection{Tiny-ImageNet}
This part describes the hyperparameter settings used in the Tiny-ImageNet experiments, offering comprehensive details to facilitate reproducibility for future studies.

\noindent \textbf{Pre-training phase.} Table~\ref{tab:squeeze_tiny} presents a detailed overview of the hyperparameters used for model training on the original Tiny-ImageNet dataset.

\begin{table}[htbp!]
\centering
\renewcommand{\arraystretch}{1} 
\begin{tabular}{cc} 
\toprule
\multicolumn{2}{c}{\textbf{Hyperparameters for Model Pre-training}} \\
\midrule
Optimizer & SGD\\
Learning Rate & 0.1 \\
Weight Decay & 1e-4 \\
Momentum & 0.9 \\
Batch Size & 64\\
Epoch & 150 \\
Scheduler & Cosine Annealing \\
Augmentation & RandomCrop, Horizontal Flip \\
Loss Function & Cross-Entropy \\
\bottomrule
\end{tabular}
\vspace{1em}
\caption{Hyperparameters for Tiny-ImageNet Pre-trained Models.} 
\label{tab:squeeze_tiny}
\end{table}

\noindent \textbf{Evaluation Phase.} Table~\ref{tab:validate_tiny} details the hyperparameter settings used during the post-evaluation phase on the Distilled Tiny-ImageNet dataset.

\begin{table}[htbp!]
\centering
\renewcommand{\arraystretch}{1.3}  
\resizebox{0.6\columnwidth}{!}{
\begin{tabular}{p{3.5cm} p{4.5cm}}
\toprule
\multicolumn{2}{c}{\textbf{Hyperparameters for Post-Eval on R18, R50 and R101}} \\
\midrule
Optimizer      & Adamw                         \\ 
S1  &   IPC50 (R18)     \\ 
S2  &  IPC1 (R18) IPC10 (R18)     \\ 
S3  &    IPC50 (R50, R101)    \\ 
S4  &   IPC1 (R50, R101) IPC10 (R50, R101)      \\ 
Soft Label Generation & BSSL \\
Loss Function  & KL-Divergence                \\ 
Batch Size     & 16                          \\ 
Epochs         & 300 (IPC10, IPC50), 1000 (IPC1)                      \\
Augmentation   & RandomResizedCrop, \newline Horizontal Flip, CutMix \\
\bottomrule
\end{tabular}
}
\vspace{1em}
\caption{Hyperparameters for post-evaluation task on ResNet18, ResNet50 and ResNet101 for Tiny-ImageNet.}
\label{tab:validate_tiny}
\end{table}

\subsection{ImageNette}
This subsection describes the hyperparameter settings utilized in the ImageNette experiments, offering detailed information to facilitate reproducibility for subsequent studies.

\noindent \textbf{Pre-training phase.} Table~\ref{tab:squeeze_imagenette} summarizes the hyperparameters used for training models on the original ImageNette dataset to generate the distilled dataset, ensuring clarity and reproducibility.

\begin{table}[htbp!]
\centering
\renewcommand{\arraystretch}{1}  
\begin{tabular}{cc} 
\toprule
\multicolumn{2}{c}{\textbf{Hyperparameters for Model Pre-training}} \\
\midrule
Optimizer & SGD\\
Learning Rate & 0.01 \\
Weight Decay & 1e-4 \\
Momentum & 0.9 \\
Batch Size & 128\\
Epoch & 300 \\
Scheduler & Cosine Annealing \\
Augmentation & RandomReizeCrop, Horizontal Flip \\
Loss Function & Cross-Entropy \\
\bottomrule
\end{tabular}
\vspace{1em}
\caption{Hyperparameters for ImageNette Pre-trained Models.} 
\label{tab:squeeze_imagenette}
\end{table}

\noindent \textbf{Evaluation Phase.} Table~\ref{tab:validate_imagenette} details the hyperparameter settings applied during the post-evaluation phase on the Distilled ImageNette dataset.

\begin{table}[htbp!]
\centering
\renewcommand{\arraystretch}{1}  
\resizebox{0.6\columnwidth}{!}{
\begin{tabular}{p{3.5cm} p{4.5cm}}
\toprule
\multicolumn{2}{c}{\textbf{Hyperparameters for Post-Eval on R18, R50 and R101}} \\
\midrule
Optimizer      & Adamw                         \\ 
S2  &     IPC50 (R101)         \\ 
S3  &    IPC10 (R18, R50)   IPC50(R50)   \\ 
S4  &  IPC1(R18, R50, R101) IPC10 (R101) IPC50 (R18)    \\ 
Soft Label Generation & BSSL \\
Loss Function  & KL-Divergence                \\ 
Batch Size     & 16                          \\ 
Epochs         & 300                    \\
Augmentation   & RandomResizedCrop, \newline Horizontal Flip, CutMix \\
\bottomrule
\end{tabular}
}
\vspace{1em}
\caption{Hyperparameters for post-evaluation task on ResNet18, ResNet50 and ResNet101 for ImageNette.}
\label{tab:validate_imagenette}
\end{table}

\subsection{ImageWoof}
This section describes the hyperparameter settings used in the ImageWoof experiments, offering detailed information to facilitate reproducibility for future studies.

\noindent \textbf{Pre-training phase.} Table~\ref{tab:squeeze_imagewoof} presents a detailed overview of the hyperparameters utilized for training models on the original ImageWoof dataset to produce the distilled dataset.

\begin{table}[htbp!]
\centering
\renewcommand{\arraystretch}{1}  
\begin{tabular}{cc} 
\toprule
\multicolumn{2}{c}{\textbf{Hyperparameters for Model Pre-training}} \\
\midrule
Optimizer & SGD\\
Learning Rate & 0.1 \\
Weight Decay & 1e-4 \\
Momentum & 0.9 \\
Batch Size & 128\\
Epoch & 50 \\
Scheduler & Cosine Annealing \\
Augmentation & RandomResizeCrop, Horizontal Flip \\
Loss Function & Cross-Entropy \\
\bottomrule
\end{tabular}
\vspace{1em}
\caption{Hyperparameters for ImageWoof Pre-trained Models.} 
\label{tab:squeeze_imagewoof}
\end{table}

\noindent \textbf{Evaluation Phase.} Table~\ref{tab:validate_imagewoof} presents the hyperparameter settings utilized during the post-evaluation stage on the Distilled Imagewoof dataset, detailing the configurations applied for performance assessment.

\begin{table}[htbp!]
\centering
\renewcommand{\arraystretch}{1}  
\resizebox{0.6\columnwidth}{!}{
\begin{tabular}{p{3.5cm} p{4.5cm}}
\toprule
\multicolumn{2}{c}{\textbf{Hyperparameters for Post-Eval on R18, R50 and R101}} \\
\midrule
Optimizer      & Adamw                         \\ 
S1  &     IPC1 (R101)           \\ 
S2  &    IPC50 (R18)      \\ 
S3  &   IPC10 (R18, R50)  IPC50 (R50, R101)     \\ 
S4  &   IPC1 (R18, R50) IPC10 (R101)        \\ 

Soft Label Generation & BSSL \\
Loss Function  & KL-Divergence                \\ 
Batch Size     & 16                          \\ 
Epochs         & 300                       \\
Augmentation   & RandomResizedCrop, \newline Horizontal Flip, CutMix \\
\bottomrule
\end{tabular}
}
\vspace{1em}
\caption{Hyperparameters for post-evaluation task on ResNet18, ResNet50 and ResNet101 for ImageWoof.}
\label{tab:validate_imagewoof}
\end{table}

\subsection{ImageNet-1k}
This subsection outlines the hyperparameter configurations employed in the ImageNet1k experiments, providing the necessary details to ensure reproducibility in future research.

\noindent \textbf{Pre-training phase.} For ImageNet-1K, we employed the official PyTorch pretrained models, which have been extensively trained on the full ImageNet-1K dataset. 

\begin{table}[htbp!]
\centering
\renewcommand{\arraystretch}{1.2}  
\resizebox{0.6\columnwidth}{!}{
\begin{tabular}{p{3.5cm} p{4.5cm}}
\toprule
\multicolumn{2}{c}{\textbf{Hyperparameters for Post-Eval on R18, R50 and R101}} \\
\midrule
Optimizer      & Adamw                         \\ 
S1  &       IPC50 (R18, R50)        \\ 
S2  &      IPC1 (R18) IPC10 (R18, R50, R101)       \\ 
S3  &    IPC50 (R101)    \\ 
S4  &     IPC1 (R50, R101)    \\ 
Soft Label Generation & BSSL \\
Loss Function  & KL-Divergence                \\ 
Batch Size     & 16                          \\ 
Epochs         & 300                       \\
Augmentation   & RandomResizedCrop, \newline Horizontal Flip, CutMix \\
\bottomrule
\end{tabular}
}
\vspace{1em}
\caption{Hyperparameters for post-evaluation task on ResNet18, ResNet50 and ResNet101 for ImageNet-1k.}
\label{tab:validate_imagenet-1k}
\end{table}

\noindent \textbf{Evaluation Phase.} Table~\ref{tab:validate_imagenet-1k} provides a detailed overview of the hyperparameter settings used during the post-evaluation phase on the Distilled ImageNet-1k dataset.

\section{Additional Distilled Data Visualization}
Additional visualizations of the distilled data generated by \name{} are provided in Fig.~\ref{fig:cifar100_vis} (CIFAR-100), Fig.~\ref{fig:tiny_vis} (Tiny-ImageNet), Fig.~\ref{fig:imageNette_vis} (ImageNette), Fig.~\ref{fig:imagewoof_vis} (ImageWoof), and Fig.~\ref{fig:imageNet1k_vis} (ImageNet-1K). Furthermore, enhanced versions - \name{}+ are presented in Fig.~\ref{fig:cifar100_vis+} (CIFAR-100), Fig.~\ref{fig:tiny_vis+} (Tiny-ImageNet), Fig.~\ref{fig:imageNette_vis+} (ImageNette), Fig.~\ref{fig:imagewoof_vis+} (ImageWoof), and Fig.~\ref{fig:imageNet1k_vis+} (ImageNet-1K).

\begin{figure*}[ht]
    \centering
    \includegraphics[width=1\textwidth]{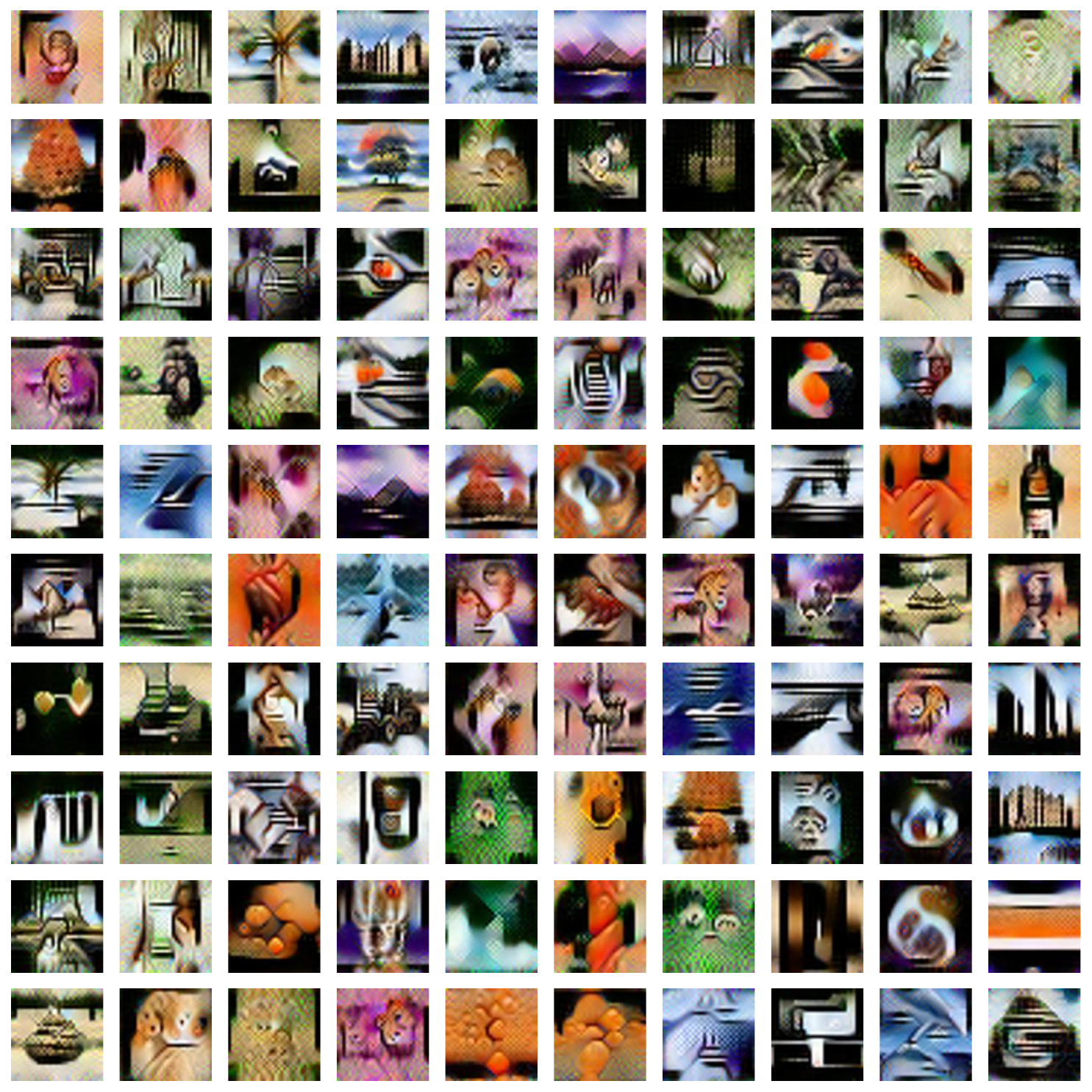}
    \caption{Visualization of synthetic data on CIFAR-100 generated by \name{}.}
    \label{fig:cifar100_vis}
\end{figure*}

\begin{figure*}[ht]
    \centering
    \includegraphics[width=1\textwidth]{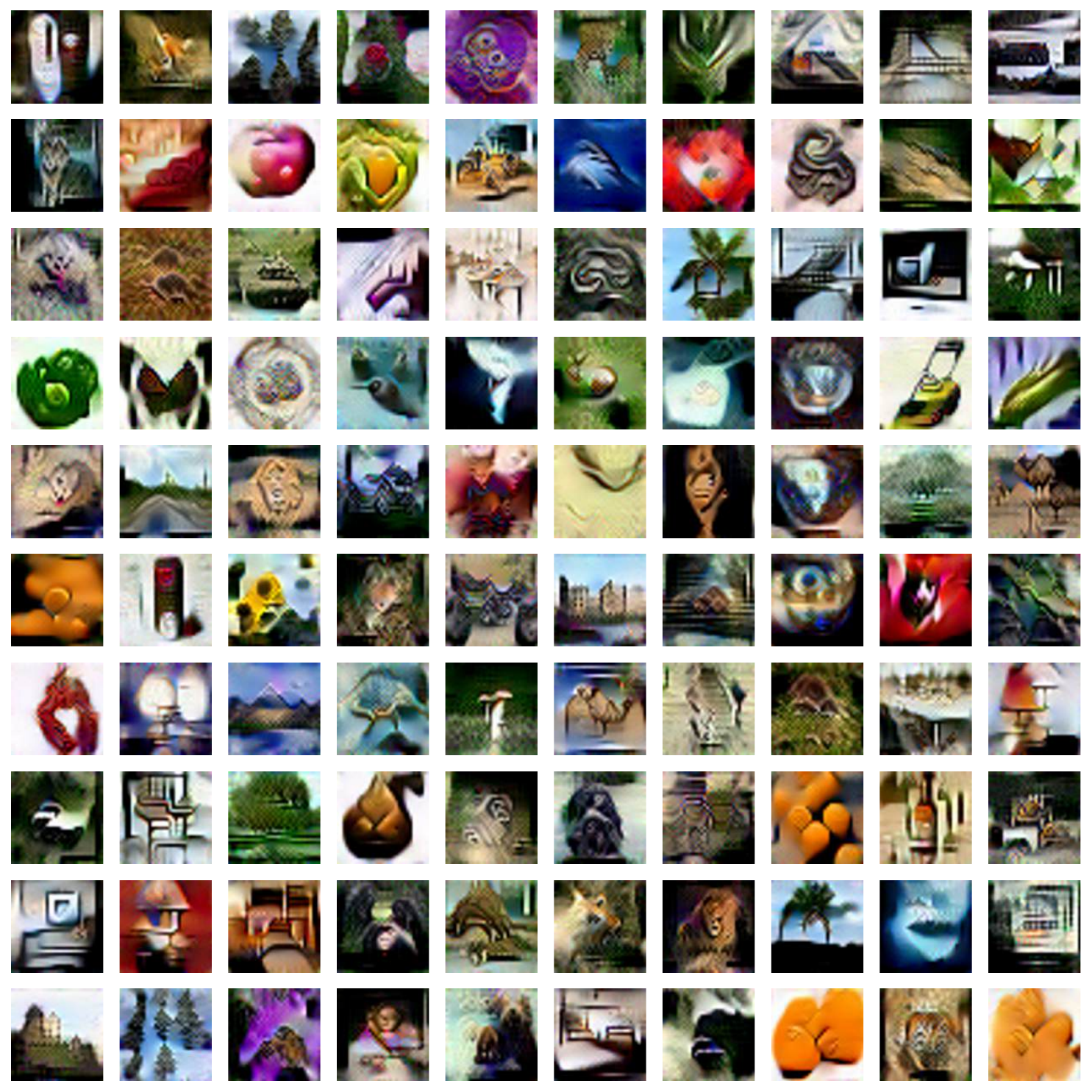}
    \caption{Visualization of synthetic data on CIFAR-100 generated by \name{}+.}
    \label{fig:cifar100_vis+}
\end{figure*}

\begin{figure*}[ht]
    \centering
    \includegraphics[width=1\textwidth]{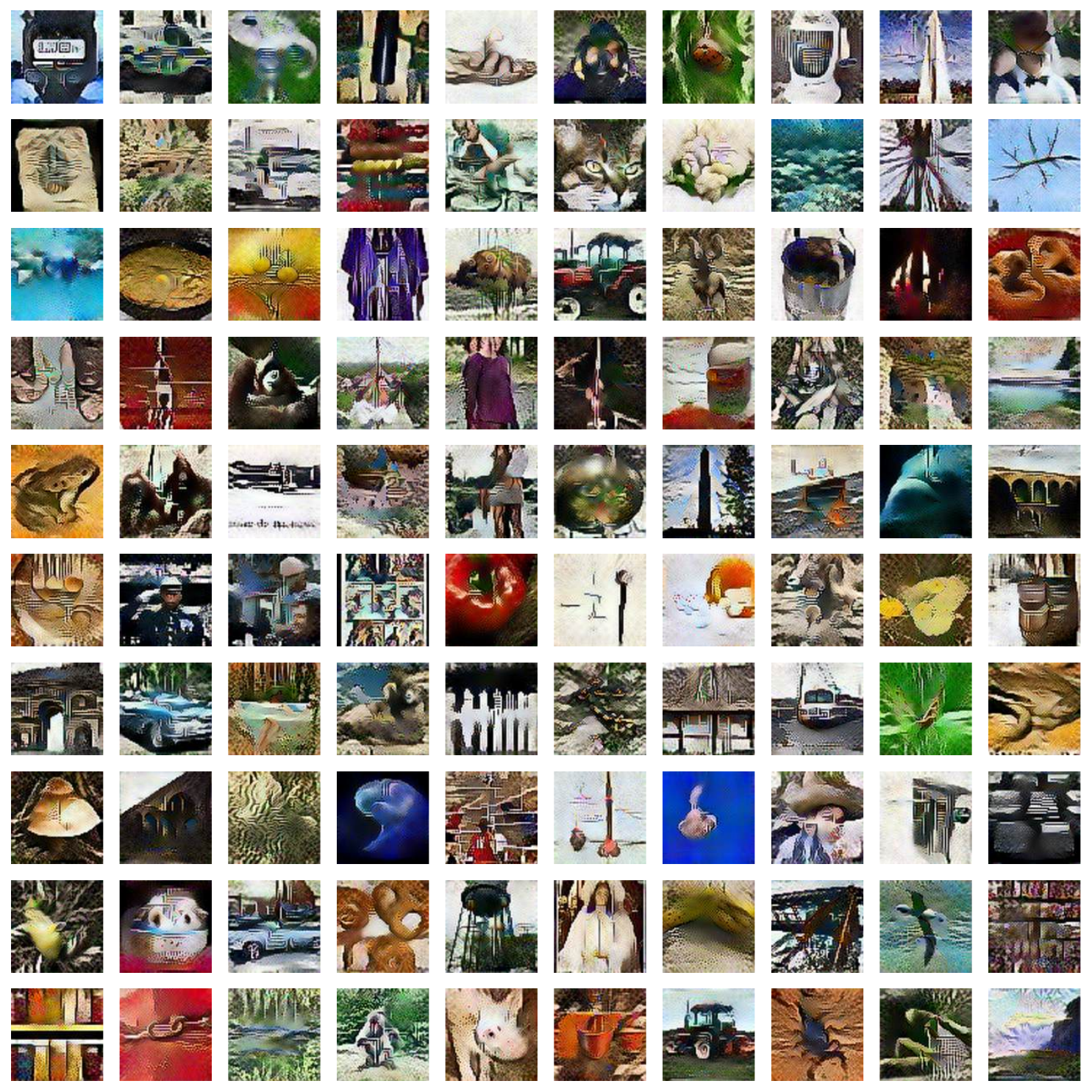}
    \caption{Visualization of synthetic data on Tiny-ImageNet generated by \name{}.}
    \label{fig:tiny_vis}
\end{figure*}

\begin{figure*}[ht]
    \centering
    \includegraphics[width=1\textwidth]{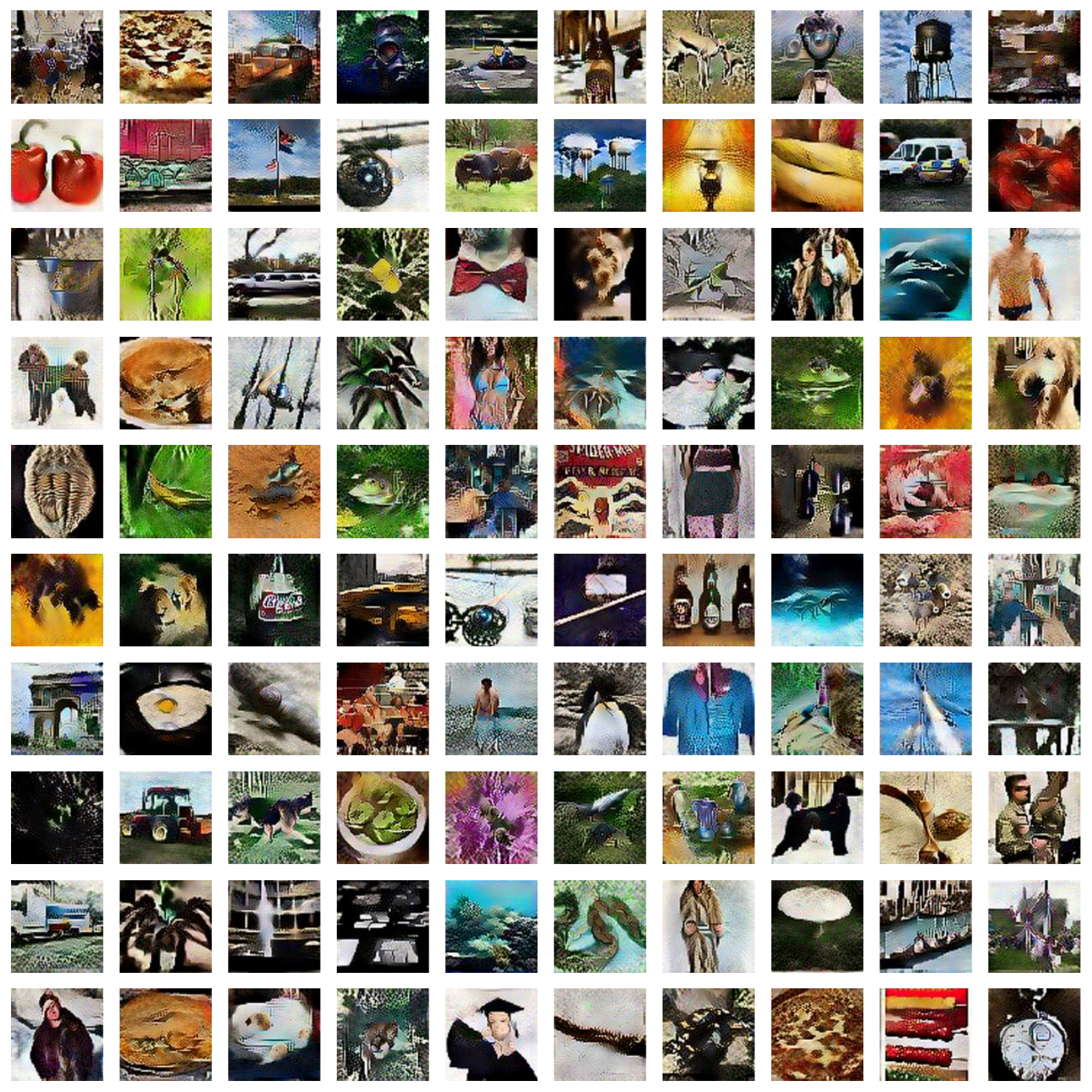}
    \caption{Visualization of synthetic data on Tiny-ImageNet generated by \name{}+.}
    \label{fig:tiny_vis+}
\end{figure*}

\begin{figure*}[ht]
    \centering
    \includegraphics[width=1\textwidth]{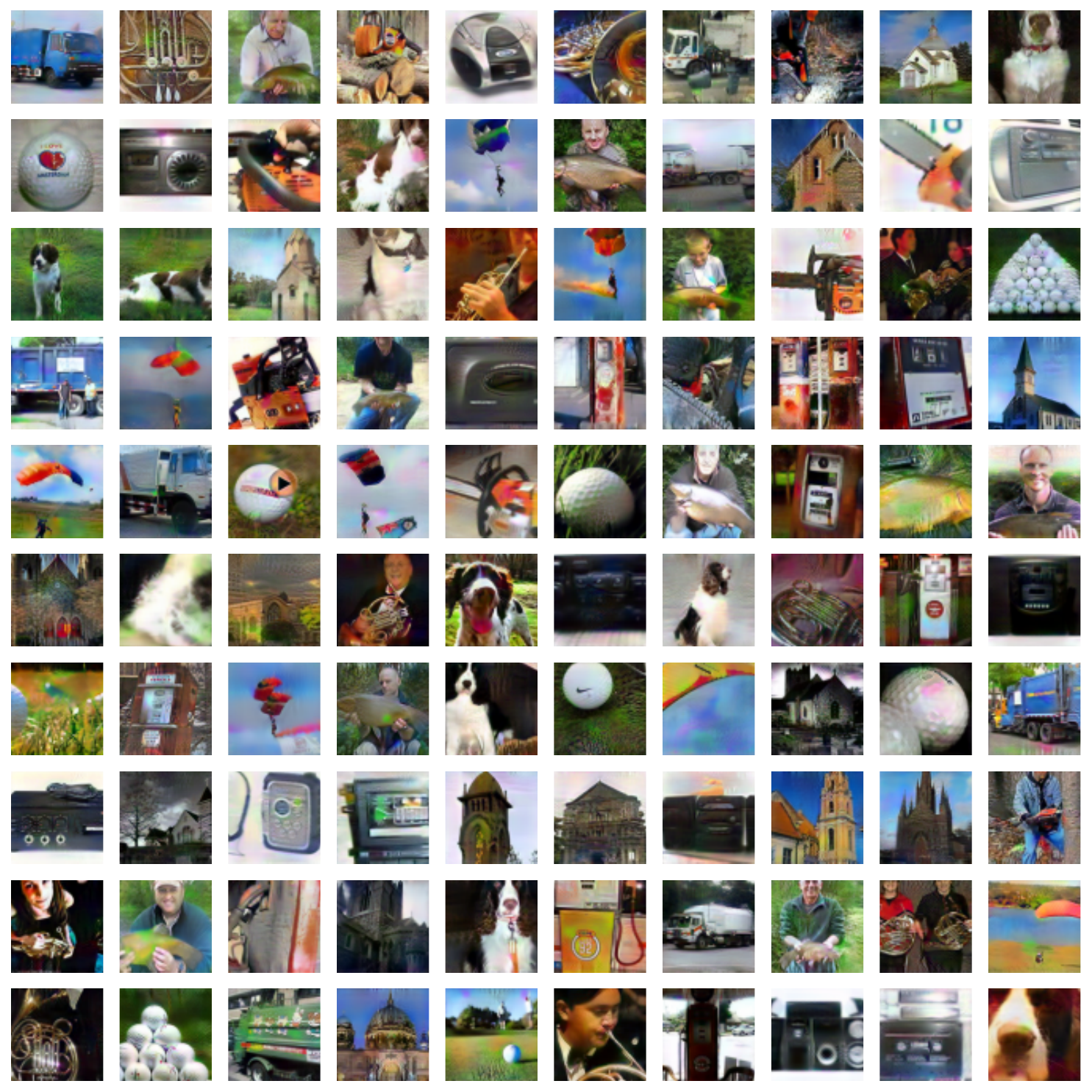}
    \caption{Visualization of synthetic data on ImageNette generated by \name{}.}
    \label{fig:imageNette_vis}
\end{figure*}
\begin{figure*}[ht]
    \centering
    \includegraphics[width=1\textwidth]{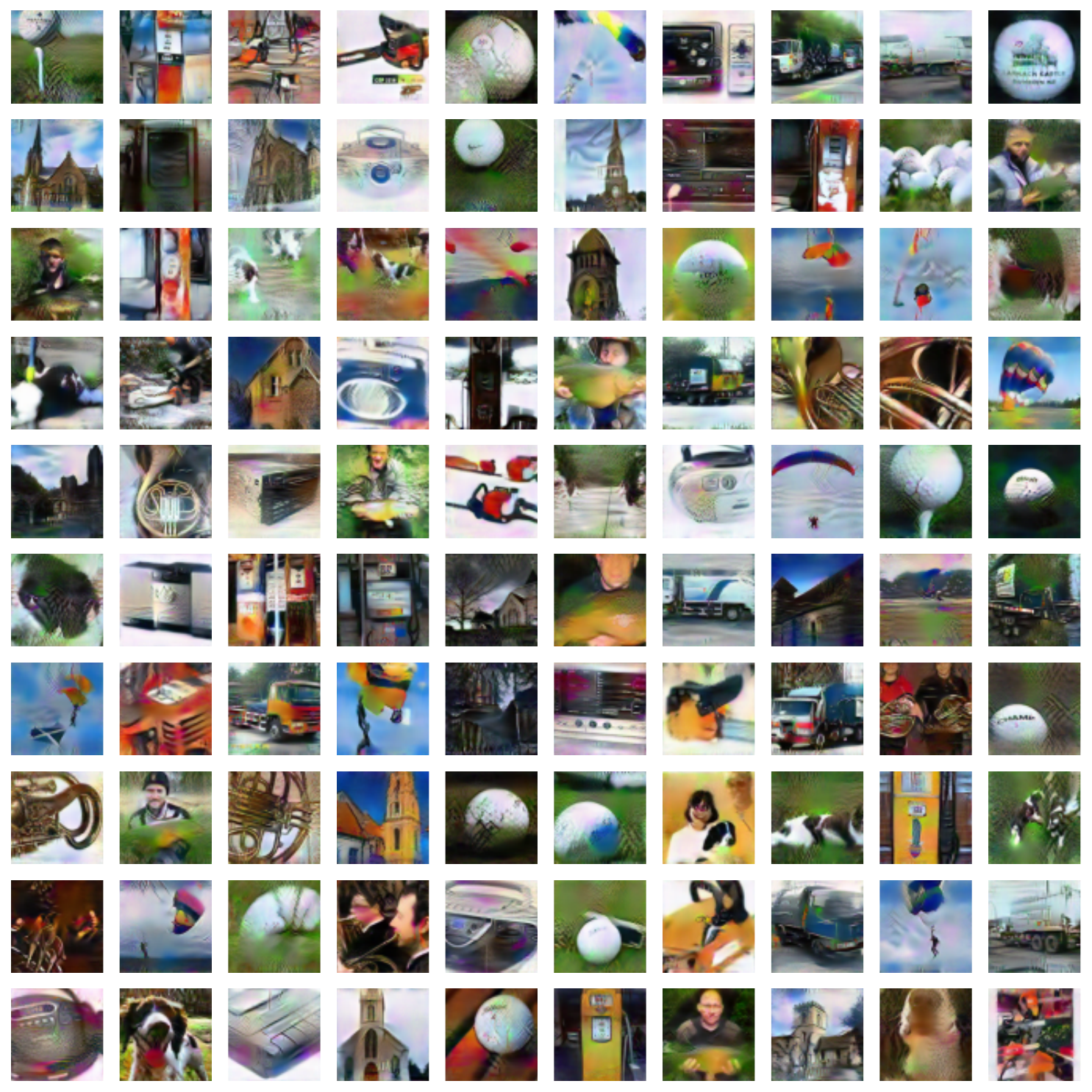}
    \caption{Visualization of synthetic data on ImageNette generated by \name{}+.}
    \label{fig:imageNette_vis+}
\end{figure*}

\begin{figure*}[ht]
    \centering
    \includegraphics[width=1\textwidth]{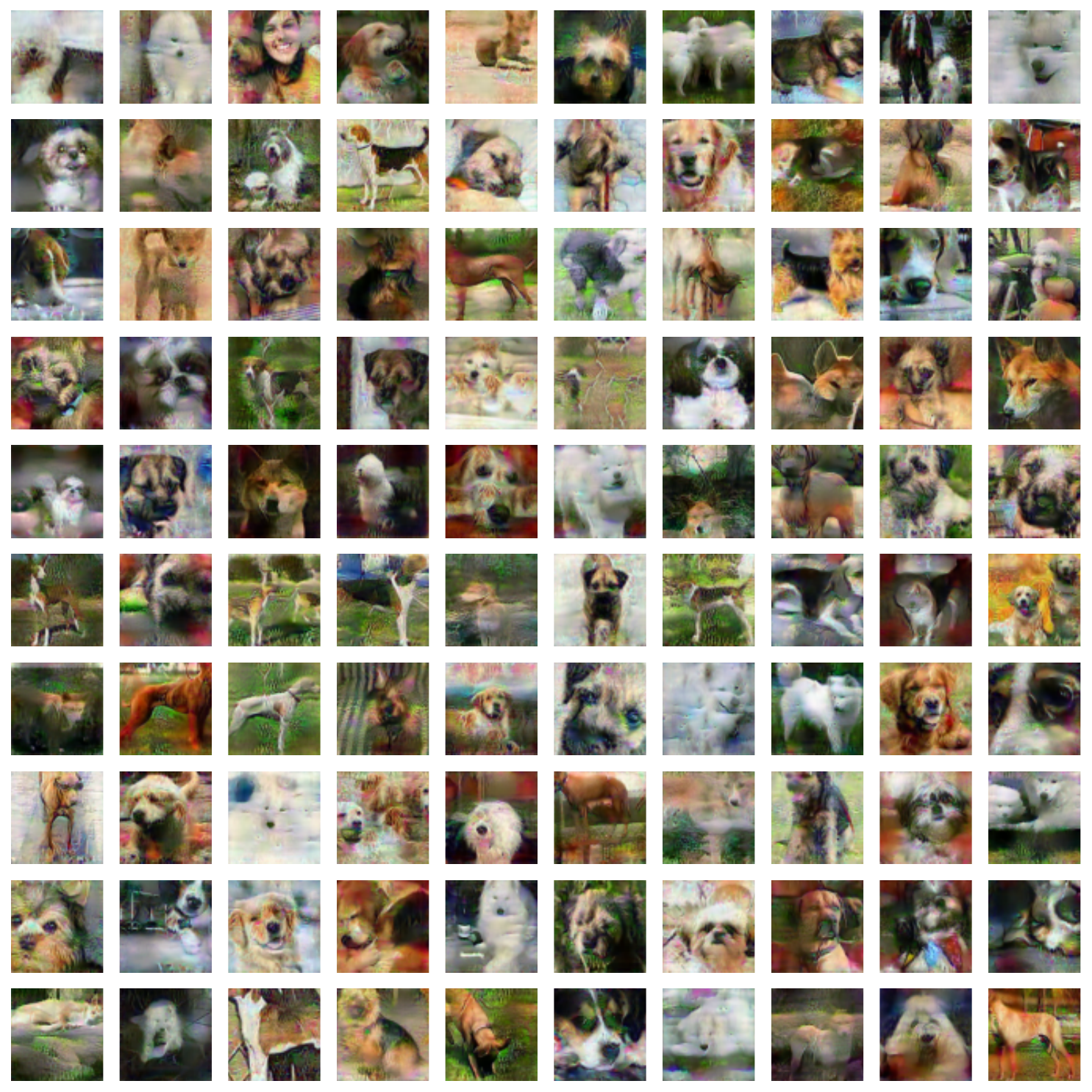}
    \caption{Visualization of synthetic data on ImageWoof generated by \name{}.}
    \label{fig:imagewoof_vis}
\end{figure*}
\begin{figure*}[ht]
    \centering
    \includegraphics[width=1\textwidth]{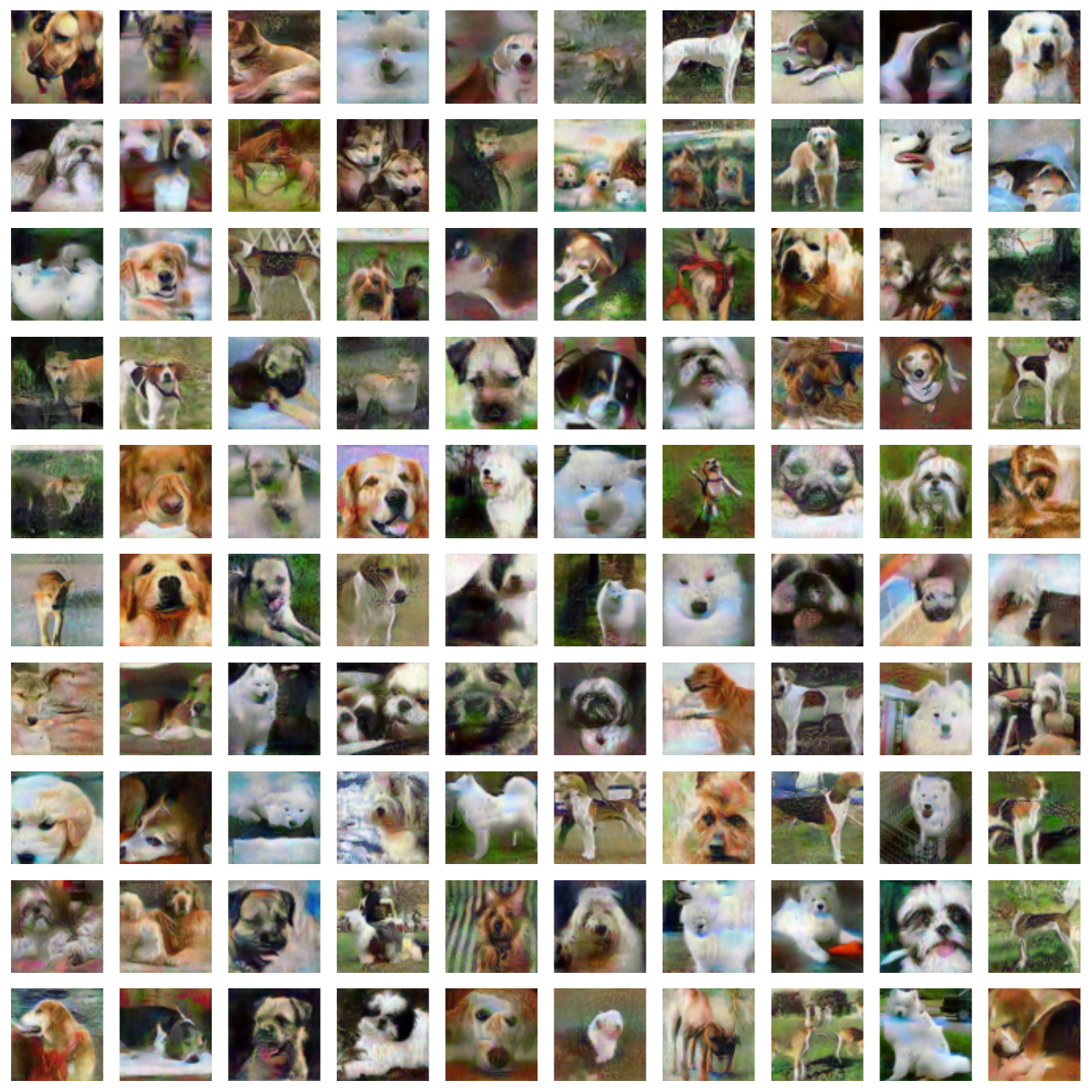}
    \caption{Visualization of synthetic data on ImageWoof generated by \name{}+.}
    \label{fig:imagewoof_vis+}
\end{figure*}

\begin{figure*}[ht]
    \centering
    \includegraphics[width=1\textwidth]{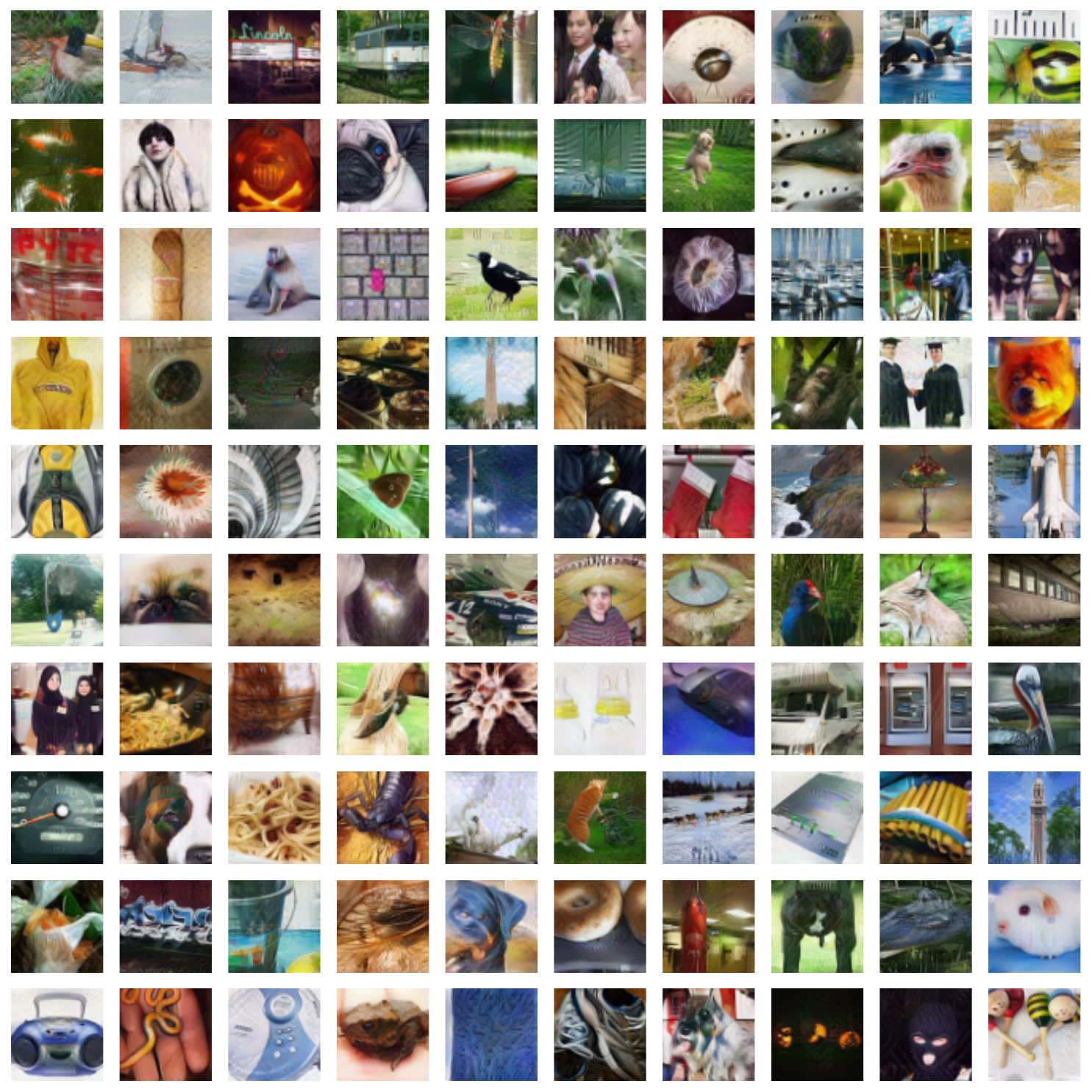}
    \caption{Visualization of synthetic data on ImageNet-1k generated by \name{}.}
    \label{fig:imageNet1k_vis}
\end{figure*}
\begin{figure*}[ht]
    \centering
    \includegraphics[width=1\textwidth]{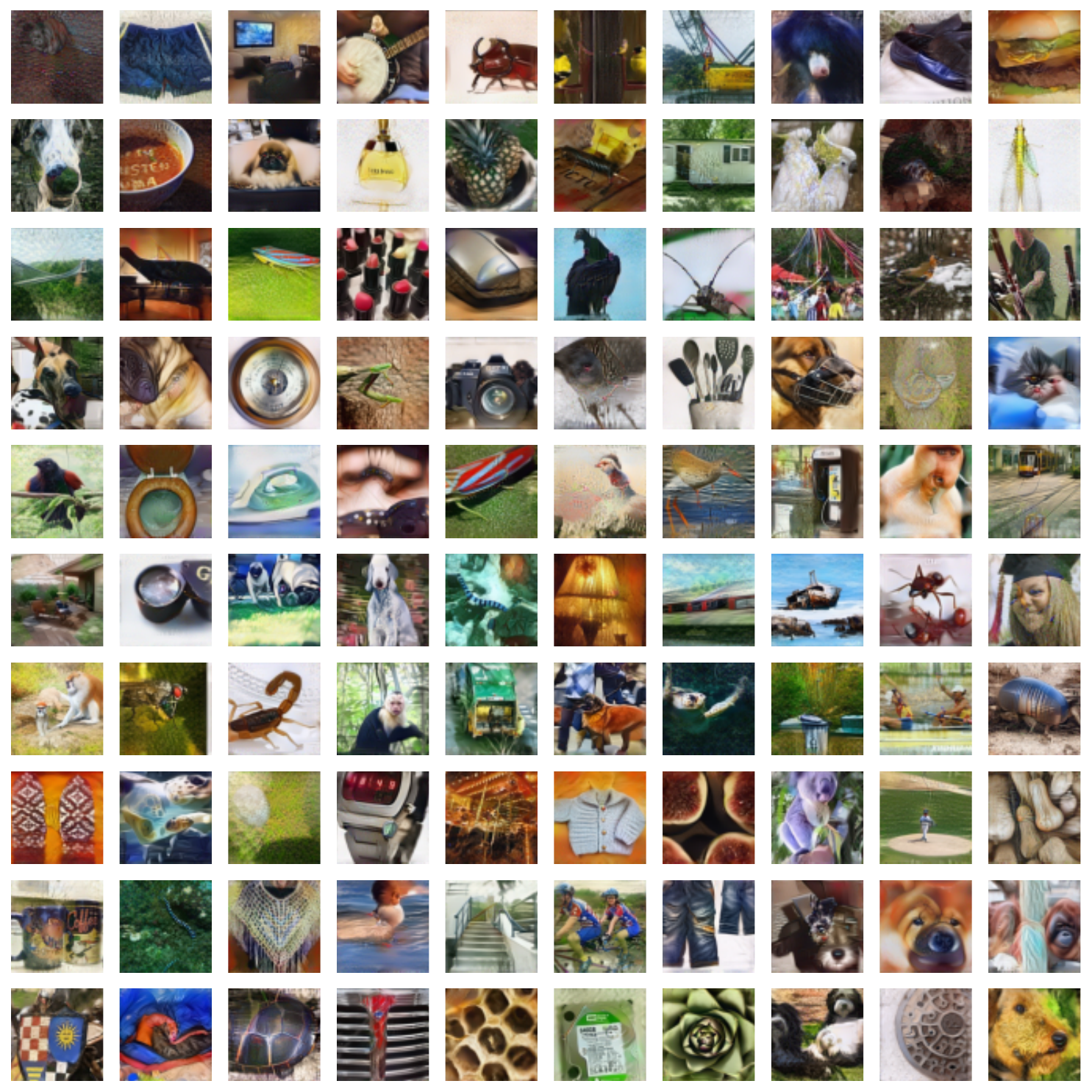}
    \caption{Visualization of synthetic data on ImageNet-1k generated by \name{}+.}
    \label{fig:imageNet1k_vis+}
\end{figure*}

\end{document}